\documentclass[12pt]{article}
\usepackage{booktabs} 
\usepackage{amsmath,amsthm,amssymb}
\usepackage{algorithm}
\usepackage{algpseudocode}
\usepackage{graphicx}
\usepackage{subfig}
\theoremstyle{plain}
\newtheorem{theorem}{Theorem}
\newtheorem{lemma}[theorem]{Lemma}

\theoremstyle{definition}

\usepackage{paralist}
\usepackage{color}
\usepackage{verbatim} 

\newcommand{\R}{\mathbb{R}}
\newcommand{\eps}{\epsilon}

\newcommand{\aset}[1]{\{ #1 \}}

\newcommand{\rnote}[1]{}

\newcommand{\transpose}{{\rm T}}
\newcommand{\tran}{\transpose}

\newcommand{\twonorm}[1]{\left\lVert #1 \right\rVert_{2}}
\newcommand{\twotwonorm}[1]{\left\lVert #1 \right\rVert_{2}^2}
\newcommand{\fronorm}[1]{\left\lVert #1 \right\rVert_{F}}
\newcommand{\tP}{{\tilde P}}

\newcommand{\tU}{\tilde{U}}

\newcommand{\tV}{\tilde{V}}

\newcommand{\tsig}{\tilde{\sigma}}

\newcommand{\tA}{\tilde{A}}
\newcommand{\tZ}{\tilde{Z}}
\newcommand{\tq}{\tilde{q}}
\newcommand{\tp}{\tilde{p}}

\newcommand{\cN}{\mathcal{N}}

\graphicspath{{fig/}}

\begin{document}
\title{Subspace Approximation for Approximate Nearest Neighbor Search in NLP}

\author{Jing Wang \thanks{jw998@rutgers.edu}}

\maketitle

\begin{abstract}
	Most natural language processing tasks can be formulated as the approximated nearest neighbor search problem, such as word analogy, document similarity, machine translation. Take the question-answering task as an example, given a question as the query, the goal is to search its nearest neighbor in the training dataset as the answer. However, existing methods for approximate nearest neighbor search problem may not perform well owing to the following practical challenges: 1) there are noise in the data; 2) the large scale dataset yields a huge retrieval space and high search time complexity. 
	
	In order to solve these problems, we propose a novel approximate nearest neighbor search framework which i) projects the data to a subspace based spectral analysis which eliminates the influence of noise; ii) partitions the training dataset to different groups in order to reduce the search space. Specifically, the retrieval space is reduced from $O(n)$ to $O(\log n)$ (where $n$ is the number of data points in the training dataset). We prove that the retrieved nearest neighbor in the projected subspace is the same as the one in the original feature space. We demonstrate the outstanding performance of our framework on real-world natural language processing tasks.
\end{abstract}

\section{Introduction}
Artificial intelligence (AI) is a thriving field with active research topics and practical products, such as Amazon's Echo, Goolge's home smart speakers and Apple's Siri. The world becomes enthusiastic to communicate with these intelligent products. Take Amazon Echo for example,  you can ask Echo the calories of every food in your plate if you are on diet. Whenever you need to check your calendar, just ask ``Alexa, what's on my calendar today?'' It boosts the development of natural language processing which refers to the AI technology that makes the communication between AI products and humans with human language possible. It is shown that the communication between human and AI products are mainly in the format of question answering (QA). QA is a complex and general natural language task. Most of natural language processing tasks can be treated as question answering problem, such as word analogy task \cite{mikolov2013linguistic}, machine translation \cite{wu2016google}, named entity recognition (NER) \cite{liu2011recognizing,passos2014lexicon}, part-of-speech tagging (POS) \cite{kumar2016ask}, sentiment analysis \cite{socher2013recursive}.

There are many works designed for the question answering task, such as deep learning models \cite{kumar2016ask}
, information extraction systems \cite{yates2007textrunner}. In this work, we propose to solve the question answering task by the approximate nearest neighbor search method.  Formally, given the question as a query $q\in R^d$, the training data set $P=\{p_1,\cdots,p_n\}\in R^{d\times n}$, the nearest neighbor search aims to retrieve the nearest neighbor of the query $q$, denoted as $p^{\ast}$ from $P$ as the answer. We assume that $p^{\ast}$ is within distance 1 from the query $q$, and all other points are at distance at least 1+ $\epsilon$ ($\epsilon \in (0,1)$) from the query $q$. The nearest neighbor $p^{\ast}$ is called a $(1+\epsilon)$-approximate nearest neighbor to $q$ which can be expressed as:
\begin{equation} \label{eq:pstarAdv}
\begin{split}
&	\exists ~p^*\in P,~
\twonorm{q-p^*}\leq 1 \text { and }\\
&\forall~ p\in P\setminus\{p^*\},\ 
\twonorm{q-p} \geq 1+\eps
\end{split}
\end{equation}
However, in real-world natural language processing applications, there are usually noise in the data, such as spelling errors, non-standard words in newsgroups, pause filling words in speech. Hence, we assume the data set $P$ with arbitrary noise $t$ to create $\tilde{P} =\{ \tilde{p}_1,\cdots,\tilde{p}_n\}$, where $\tilde{p}_i = p_i+t_i$. The query $q$ is perturbed similarly to get $\tilde{q}=q+t_q$. We assume that the noise is bounded, that is $\twonorm{t_i} \leq \epsilon/25$ and $\twonorm{t_q}\leq \epsilon/25$.  

There are many approaches proposed to solve the approximate nearest neighbor search problem. Existing methods can be classified as two groups: the data-independent methods and the data-dependent methods. The data-independent approaches are mainly based on random projection to get data partitions, such as Local Sensitive Search, Minwise Hashing. Recently,  the data-dependent methods received interest for its outstanding performance. They utilize spectral decomposition to map the data to different subspace, such as Spectral hashing. However, theoretical guarantee about the performance is not provided. Existing methods can not handle natural language processing problems well for the following reasons. First the data set in natural language processing is usually in large scale which yields a huge search space. Moreover, the singular value decomposition which is widely used to obtain the low-rank subspace is too expensive here.  Second, the data is with noise. Existing data-aware projection is not robust to noisy data which cannot lead to correct partitions. 

To solve the above mentioned problem, we propose a novel iterated spectral based approximate nearest neighbor search framework for general question answering tasks (Random Subspace based Spectral Hashing (RSSH)). Our framework consists of the following major steps:
\begin{itemize}
	\item As the data is with noise, we first project the data to the clean low-rank subspace. We obtain a low-rank approximation within (1+$\delta$) of optimal for spectral norm error by the randomized block Krylov methods which enjoys the time complexity $O(nnz(X))$ \cite{musco2015randomized}. 
	\item To eliminate the search space, we partition data to different clusters. With the low-rank subspace approximation, data points with are clustered corresponding to their distance to the subspace. 
	\item Given the query, we first locate its nearest subspace and then search the nearest neighbor in the data partition set corresponding to the nearest subspace. 
\end{itemize}
With our framework, we provide theoretical guarantees in the following ways:
\begin{itemize}
	\item With the low-rank approximation, we prove that the noise in the projected subspace is small. 
	\item With the data partition strategy, all data will fall to certain partition within $O(\log n)$ iterations.
	\item We prove that our method can return the nearest neighbor of the query in low-rank subspace which is the nearest neighbor in the clean space. 
\end{itemize}

To the best of our knowledge, it is the first attempt of spectral nearest neighbor search for question answering problem with theory justification. Generally, our framework can solve word similarity task, text classification problems (sentiment analysis), word analogy task and named entity recognition problem.

The theoretical analysis in this work is mainly inspired by the work in \cite{abdullah2014spectral}. The difference is that the subspace of data sets is computed directly in \cite{abdullah2014spectral}, in our work, we approximate the subspace by a randomized variant of the Block Lanczos method \cite{musco2015randomized}. In this way, our method enjoys higher time efficiency and returns a (1+$\epsilon/5$)-approximate nearest neighbor.

\section{Notation}
In this work, we let $s_j(M)$ denote the $j$-th largest singular value
of a real matrix $M$.
$\fronorm{M}$ is used to denote the Frobenius norm of $M$.
All vector norms, i.e. $\twonorm{v}$ for $v\in\R^d$, refer to the $\ell_2$-norm.

The \emph{spectral norm} of a matrix $X\in\R^{n\times d}$ is defined
as 
\begin{equation}
\twonorm{X} = \sup_{y\in \R^d: \twonorm{y}=1} \twonorm{X y},
\end{equation}
where all vector norms $\twonorm{\cdot}$ refer throughout to the $\ell_2$-norm. It is clear that $\twonorm{M}=s_1(M)$ equals the spectral norm of $M$.
The Frobenius norm of $X$ is defined as $\fronorm{X}=(\sum_{ij}
X_{ij}^2)^{1/2}$, and let $X^\tran$ denote the transpose of $X$.
A \emph{singular vector} of $X$ is a unit vector $v \in \R^d$ 
associated with a \emph{singular value} $s\in\R$ and a unit vector 
$u \in \R^n$ such that $X v = s u$ and $u^\tran X = s v^\tran$. 
(We may also refer to $v$ and $u$ as a pair of right-singular and left-singular vectors associated with $s$.)

Let $p_{\tilde{U}}$ denote the projection of a point
$p$ onto $\tilde{U}$. Then the distance between a point $x$ and
a set (possibly a subspace) $S$ is defined as $d(x,S) = \inf_{y\in S}
\twonorm{x-y}$.

\section{Problem}
Given an $n$-point dataset $P$ and a query point $q$,
both lying in a $k$-dimensional space $U \subset \R ^d$,
we aim to find its nearest neighbor $q^*$ which satisfying that:
\begin{equation} \label{eq:pstarAdv}
\exists p^*\in P \text{ such that }
\twonorm{q-p^*}\leq 1 \text { and }
\forall p\in P\setminus\{p^*\},\ 
\twonorm{q-p} \geq 1+\eps
\end{equation}
Assume that the data points are corrupted by arbitrary small noise $t_i$ which is bounded $\twonorm{t_i} \leq \eps/16$ for all $i$ ($\eps\in(0,1)$). The observed set $\tilde P$ consists of points $\tilde p_i=p_i+t_i$ for all $p_i\in P$
and the noisy query points $\tilde q = q+t_q$ with $\twonorm{t_q} \le \eps/16$.

\section{Algorithm}

\subsection{Subspace Approximation}
We utilize a randomized variant of the Block Lanczos method proposed in \cite{musco2015randomized} to approximate the low-rank subspace of the data set. 
\begin{algorithm}[t]
	\caption{Block Lanczos method \cite{musco2015randomized}}
	\label{alg:musco}
	\begin{algorithmic}[1]
		\State \textbf{Input: } $A \in \R^{m \times d}, \text{rank} ~ k \le d, m, \text{error}~ \eta \in (0,1)$
		\State $r := \Theta \left(\frac{\log(m)}{\sqrt{\eta}} \right)$, $c=rk$, $\Pi \sim \cN(0,1)^{d \times k}$
		\State Compute $K := [A \Pi,(AA^\top) A \Pi,\ldots, (AA^\top)^{r} A \Pi]$
		\State Orthonormalize $K$'s columns to obtain $Q \in \R^{m \times rk}$ 
		\State Compute the truncated SVD $(Q^\top A)_k := W_k \Sigma_k V_k^\top\in R^{c\times c}$   
		\State Compute $Z_k := Q W_k \in R^{m\times k}$ 
		\State \textbf{Output:} $Z_k$ 
	\end{algorithmic}
\end{algorithm}
\begin{theorem} \label{thm:musco} \textbf{\cite{musco2015randomized}}
	For any $A\in R^{m\times d}$, the $k$-dimensional low-rank subspace obtained by singular value decomposition is denoted as $A_k$. Algorithm \ref{alg:musco} returns $Z_k$ which forms the low-rank approximation $\tA_k = Z_k Z_k^\top A$, then the following bounds hold with probability at least $9/10$:
	\begin{equation}
	\|A - \tA_k\|_2 \le (1+\eta) \|A-A_k\|_2  \le (1+\eta) \sigma_k
	\end{equation}
	\begin{align} 
	\forall i \in [k],~~|z_i^\top AA^\top z_i - u_i^\top AA^\top u_i| &=  |\tsig_i^2-\sigma_i^2 | \notag \\
	&\le \eta \sigma_{k+1}^2 ~. \label{eq:perVec}
	\end{align}
	where $\tsig_i$ is the $i$-th singular value of $A$.
	The runtime of the algorithm is $O\left(\frac{md k \log(m)}{ \sqrt{\eta}}+ \frac{k^2(m+d)}{\eta} \right)$.
\end{theorem}
Algorithm \ref{alg:musco} returns the matrix $Z_i$ which is  the approximation to the left singular vectors of data matrix $A$. We use $Z^TA$ to approximate the  right singular vectors of data matrix $A$. 

\subsection{Data Partition}

\begin{lemma}\label{cl:boundednearest} \cite{abdullah2014spectral}
	The nearest neighbor of $\tilde{q}$ in $\tilde P$ is $\tilde{p}^*$.
\end{lemma}

\begin{algorithm} [t]
	\caption{Spectral Data Partition by Low-rank Subspace Approximation}
	\label{alg:iterPCA}
	\begin{algorithmic}[1]
		\State \textbf{Input:} $\tP\in R^{n\times d}$, \text{rank}~$k \leq n,d$, \text{error} $\epsilon,\eta \in (0,1)$, \text{threshold} $\alpha=\epsilon/25$
		\State $i:=0$, $\tP_0 := \tilde{P}$ 
		\While{$\tilde P_i \neq \emptyset$}
		\State Compute the $k$-dimensional subspace approximation $S_i$ and low-rank projection matrix $\tZ_i$ of $\tilde{P}_i$ by Algorithm \ref{alg:musco}
		\State Compute the distance between data points and subspace 
		\begin{equation*}
		d(\tilde{p_j}, S_i):=\inf_{y\in S_i}\twonorm{\tp_j-y} ~~(\tp_j \in \tP_i)
		\end{equation*}
		\State Partition data points 
		\begin{equation*}
		M_i:= \aset{ \tilde{p_j} \in \tilde{P}_i:\ d(\tilde{p_i}, S_j ) \leq \sqrt{2} (1+\eta) \alpha }
		\end{equation*}
		\State Update dataset 
		\begin{equation*}
		\tilde P_{i+1} := \{ \tP_i \setminus M_i\}
		\end{equation*}
		\State Update iteration $i=i+1$
		\EndWhile
		\State \textbf{Output:} $\mathcal{{S}} = \{\tilde{S}_0, \ldots,\tilde{S}_{i-1}  \}$, $\mathcal{{Z}} = \{\tilde{Z}_0, \ldots,\tilde{Z}_{i-1}  \}$, $\mathcal{M} =\{M_0, M_1, \ldots, M_{i-1}\}$.
	\end{algorithmic}
\end{algorithm}

\begin{lemma}
	Algorithm~\ref{alg:iterPCA} terminates within $O(\log n)$ iterations.
\end{lemma}
\begin{proof} 	Let $U$ be the $k$-dimensional subspace of $P$ with projection matrix $V$, let $\tilde{U}$ be the $k$-dimensional
	subspace of $\tilde{P}$ with projection matrix $\tV$, let $S$ be the low-rank approximation returned by Algorithm \ref{alg:musco} with projection matrix $Z_k$.
	The distance between data points and subspace is computed as:
	\begin{equation*}
	\begin{split}
	\sum_{\tilde p\in \tilde{P}} d(\tp,\tU)^2 &= \sum_{\tilde p\in \tilde{P}} \inf_{y\in \tU}||\tp-y||_2^2 \\
	&= \inf||\tP-\tU_k\tU_k^T\tP||_2^2 .
	\end{split}
	\end{equation*}
	
	\begin{equation*}
	\begin{split}
	\sum_{\tilde p\in \tilde{P}} d(\tp,S)^2 &= \sum \inf_{y\in S}||\tp-y||_2^2 \\
	&= \inf||\tP-Z_kZ_k^T\tP||_2^2. \\
	\end{split}
	\end{equation*}
	
	According to Theorem \ref{thm:musco}, we can have:
	\[
	||\tP-Z_kZ_k^\top \tP||_2 \leq (1+\eta)||\tP-U_kU_k^\top \tP||_2
	\],
	We can get 
	\begin{equation}
	\sum_{\tilde p\in \tilde{P}} d(\tp,S) \leq (1+\eta)^2 \sum_{\tilde p\in \tilde{P}} d(\tp,\tU).
	\end{equation}
	
	Since $\tilde U$ minimizes 
	the sum of squared distances from all $\tilde p\in \tilde{P}$ to $\tilde U$,
	\[
	\sum_{\tilde p\in \tilde{P}} d(\tilde p,\tilde U)^2
	\leq \sum_{\tilde p\in \tilde{P}} d(\tilde p,U)^2
	\leq \sum_{\tilde p\in \tilde{P}} \twonorm{\tilde p-p}^2
	\leq \alpha^2 n.
	\]
	Then, we can get:
	\begin{equation}
	\sum_{\tilde p\in \tilde{P}} d(\tilde p,S)^2 \leq (1+\eta)^2\alpha^2n.
	\end{equation}
	Hence, there are at most half of the points in $\tilde{P}$ with distance to
	$S$  greater than $\sqrt{2}(1+\eta)\alpha$. The set
	$M$ captures at least a half fraction of points. The
	algorithm then proceeds on the remaining set. 
	After $O(\log n)$ iterations all points of $\tilde P$ must be
	captured.
\end{proof}

\begin{lemma} \label{cl:boundedReport}
	The approximated subspace $S$ that captures $\tilde{p}^*$ returns this as the $(1+\eps/5)$-approximate nearest neighbor of $\tilde q$ (in $\tilde U$).
\end{lemma}
\begin{proof} [Proof of Lemma \ref{cl:boundedReport}]
	Fix $p\neq p^*$ that is captured by the same $S$,
	and use the triangle inequality to write 
	\begin{equation}
	\begin{split}
	\twonorm{p - \tilde{p}_{S}}
	&	\leq \twonorm{p-\tilde{p}} + \twonorm{ \tilde{p} - \tilde{p}_{S} } \\
	&	\leq \alpha + \sqrt{2}(1+\eta) \alpha
	\leq 4\alpha.
	\end{split}
	\end{equation}
	
	Similarly for $p^*$,
	$\twonorm{p^* - \tilde p^*_{S}} \leq 4\alpha$,
	and by our assumption
	$\twonorm{q - \tilde q} \leq \alpha$.
	By the triangle inequality, we get 
	\begin{equation}
	\begin{split}
	\twonorm{ \tilde q - \tilde p^*_{S} } &\leq \twonorm{\tq-p} +\twonorm{p-\tp^*_S}\\
	&\leq \twonorm{\tq-q}+\twonorm{q-p}+\twonorm{p-\tp^*_S}\\
	&\leq \twonorm{q-p}+5\alpha,
	\end{split}
	\end{equation}
	\begin{equation}
	\begin{split}
	\twonorm{ \tilde q - \tilde p^*_{S} } &\geq \twonorm{\tq-p}-\twonorm{\tp^*_S-p}\\
	&\geq \twonorm{q-p} -\twonorm{q-\tq}-\twonorm{\tp^*_S-p}\\
	&\geq \twonorm{q-p}-5\alpha,
	\end{split}
	\end{equation}
	Similarly, we bound  $\twonorm{ \tilde q - \tilde p^*_{S} } $
	\begin{align*} \label{eq:bounded1}
	\frac{ \twonorm{ \tilde q - \tilde p_{S} } }
	{ \twonorm{ \tilde q - \tilde p^*_{S} } }
	&= \frac{ \twonorm{ q-p } \pm 5\alpha  }
	{ \twonorm{ q-p^* } \pm 5\alpha } 
	=    \frac{ \twonorm{ q-p } \pm \tfrac15 \eps  }
	{ \twonorm{ q-p^* } \pm \tfrac15 \eps } \\
	& \geq \frac{\twonorm{q-p} - \tfrac15 \eps}
	{\twonorm{q-p^*} + \tfrac15 \eps}
	\geq \frac{ \twonorm{ q-p^* } + \tfrac45 \eps }
	{ \twonorm{ q-p^* } + \tfrac15 \eps }
	> 1+\tfrac15 \eps.
	\end{align*}
	By using Pythagoras' Theorem 
	(recall both $\tilde p_{S},\tilde p^*_{S}\in S$),
	\begin{align*}
	\frac{ \twotwonorm{ \tilde{q}_{S} - \tilde{p}_{S} } }
	{ \twotwonorm{ \tilde{q}_{S} - \tilde p^*_{S} } }
	= \frac{ \twotwonorm{ \tilde q - \tilde p_{S} } - \twotwonorm{ \tilde q - \tilde q_{S} }  }
	{ \twotwonorm{ \tilde q - \tilde p^*_{S} } - \twotwonorm{ \tilde q - \tilde q_{S} }  } 
	> (1+\tfrac15 \eps)^2.
	\end{align*}
	Hence, $\tilde{p}^*$ is reported by 
	the $k$-dimensional subspace it is assigned to.
\end{proof}

\section{Experiment}

In this experiment, we compare our algorithm with existing hashing algorithms. 
\subsection{Baseline algorithms}

Our comparative algorithms include state-of-the-art learning to hashing algorithm such as
\begin{itemize}
	\item Anchor Graph Hashing (AGH) \cite{liu2011hashing}
	\item Circulant Binary Embedding (CBE) \cite{YuKGC14} 
	\item Multidimensional Spectral Hashing (MDSH) \cite{WeissFT12}
	\item  Iterative quantization (ITQ) \cite{GongL11}
	\item Spectral Hashing \cite{weiss2009spectral} 
	\item Sparse Projection (SP) \cite{XiaHK015}
\end{itemize}
We refer our algorithm as Random Subspace based Spectral Hashing (RSSH). 
\subsection{Datasets.}
\begin{table}[]
	\centering
	\caption{Summary of Datasets}
	\label{tab:data}
	\resizebox{\linewidth}{!}{  
		\begin{tabular}{|l|r|r|c|c|}
			\hline
			Dataset  & \#training & \#query & \#feature &\# class \\ \hline
			MNIST   & 69,000          & 1,000          & 784        &   10       \\ \hline
			CIFAR-10 & 59,000          & 1,000          & 512         &10         \\ \hline
			COIL-20  & 20,019          & 2,000          & 1,024      &20            \\ \hline
			VOC2007 & 5,011           & 4,096         & 3,720    &20\\  \hline          
	\end{tabular}}
	\vspace{-0.1in}
\end{table}
Our experiment datasets include MNIST \footnote{{http://yann.lecun.com/exdb/mnist/}}, CIFAR-10\footnote{{https://www.cs.toronto.edu/~kriz/cifar.html/}}, COIL-20 \footnote{http://www.cs.columbia.edu/CAVE/software/softlib/coil-20.php} and the 2007 PASCAL VOC challenge dataset. 

\noindent \textbf{MNIST.}\ It is a well-known handwritten digits dataset from ``0'' to ``9''. The dataset consists of 70,000 samples in feature space of dimension 784. We split the samples to a training and a query set which containing 69,000 and 1,000 samples respectively. 

\noindent \textbf{CIFAR-10.}\ There are 60,000 image in 10 classes, such as ``horse'' and ``truck''. We use the default 59,000 training set and 1,000 testing as query set. The image is with 512 GIST feature. 

\noindent \textbf{COIL-20.}\ It is from the Columbia University Image Library which contains 20 objects. Each image is represented by a feature space 1024 dimension. For each object, we choose 60\% images for training and the others are querying.
\begin{figure}
	\centering
	\includegraphics[width=0.45\textwidth]{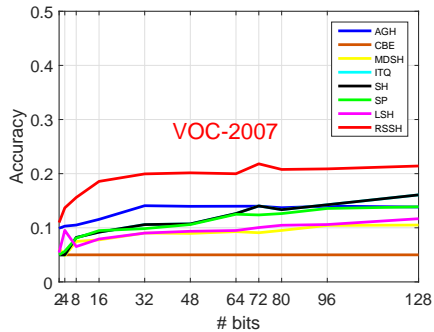}   
	\caption{The average precision on VOC2007.}
	\label{fig:avgprecision}
\end{figure}

\noindent \textbf{VOC2007.}\ The VOC2007 dataset consists of three subsets as training, validation and testing. We use the  first two subsets as training containing 5,011 samples and the other as query containing 4,096 samples. We set each image to the size of [80, 100] and extract the HOG feature with cell size 10 as their feature space \footnote{http://www.vlfeat.org/}. All the images in VOC2007 are defined into 20 subjects, such as ``aeroplane'' and ``dining tale''. For the classification task on each subject, there are 200 to 500 positive samples and the following 4,000 are negative. Thus, the label distribution of the query set is unbalanced. A brief description of the datasets are presented in Table \ref{tab:data}.

\subsection{Evaluation Metrics}
All the experiment datasets are fully annotated. We report the classification result based on the groundtruth label information. That is, the label of the query is assigned by its nearest neighbor. For the first three datasets in Table \ref{tab:data}, we report the classification accuracy. For the VOC2007 dataset, we report the precision as the label distribution is highly unbalanced. The criteria are defined in terms of true positive (TP), true negative (TN), false positive (FP) and false negative (FN) as,
\begin{equation}
\begin{split}
\textrm{Precision} &= \frac{\text{TP}}{\text{TP}+\text{FP}},\\
\text{Accuracy} &= \frac{\text{TP}+\text{TN}}{\text{TP}+\text{TN}+\text{FP}+\text{FN}}.
\end{split}
\end{equation}

For the retrieval task, we report the recall with top [1, 10, 25] retrieved samples. The true neighbors are defined by the Euclidean distance. 

We report the aforementioned evaluation criteria with varying hash bits ($r$) in the range of [2, 128].

\begin{figure*}
	\centering
	\subfloat[Recall of Top 1 retrieval]{
		\includegraphics[width=0.3\textwidth]{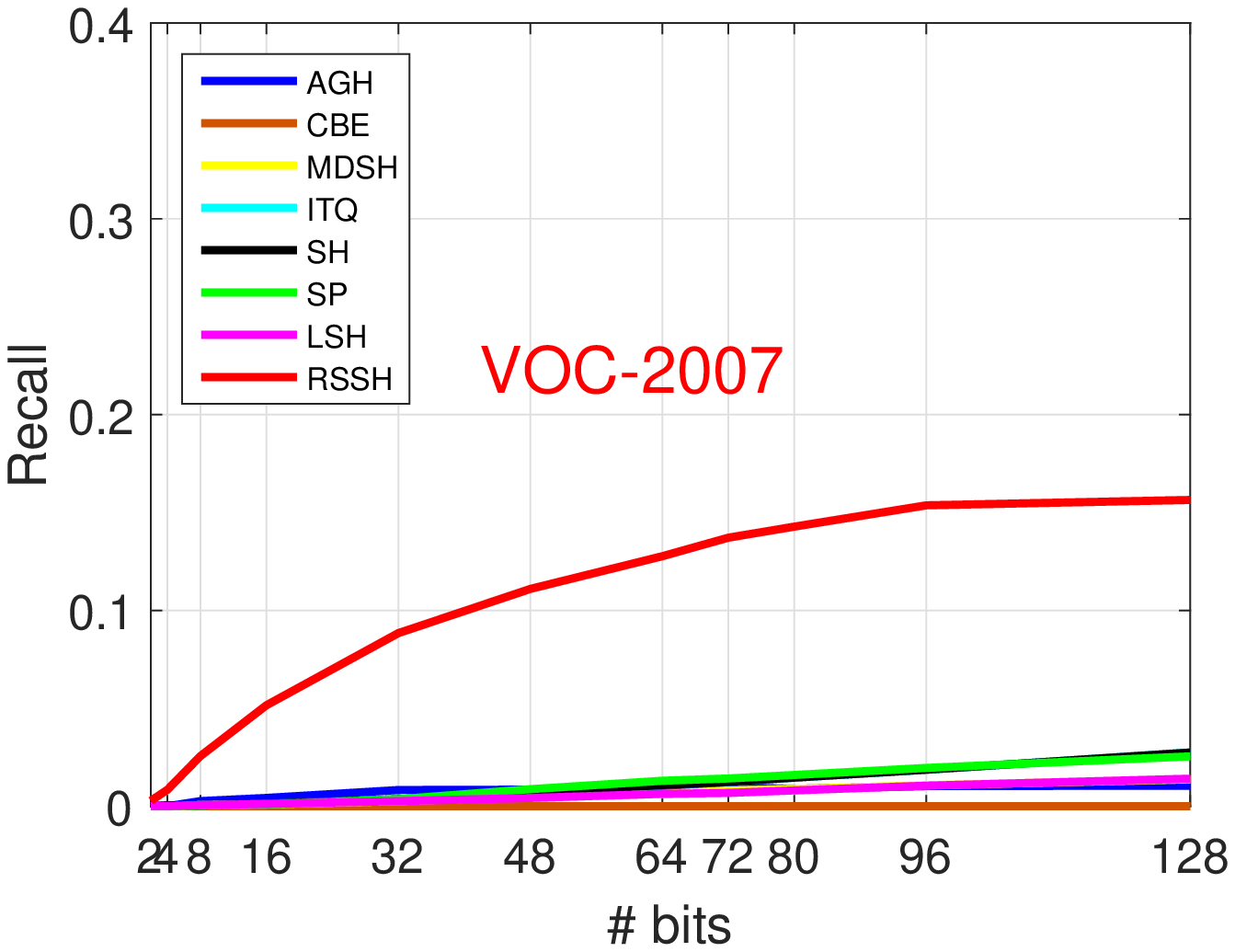}}
	\subfloat[Recall of Top 10 retrieval]{
		\includegraphics[width=0.3\textwidth]{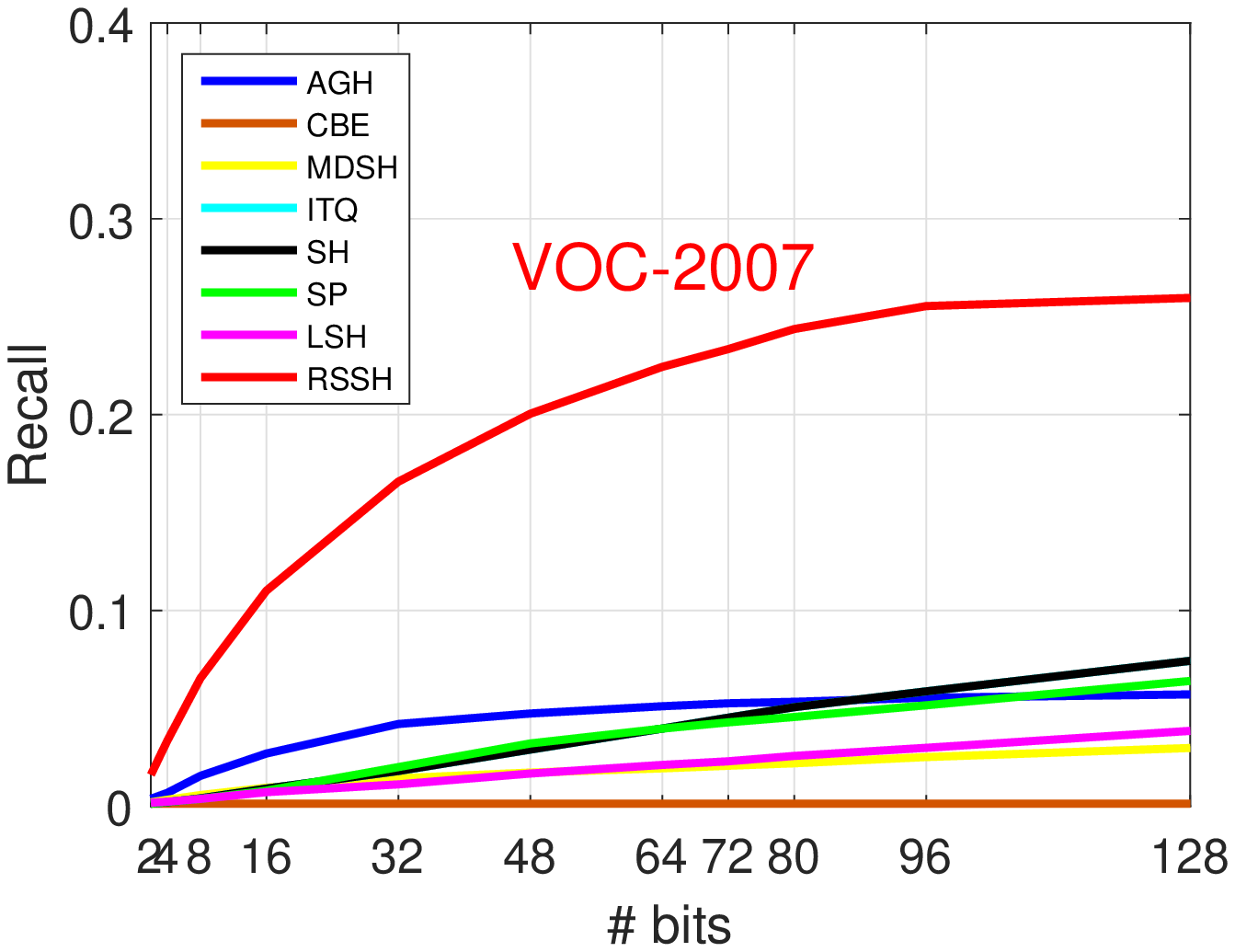}}
	\subfloat[Recall of Top 25 retrieval]{
		\includegraphics[width=0.3\textwidth]{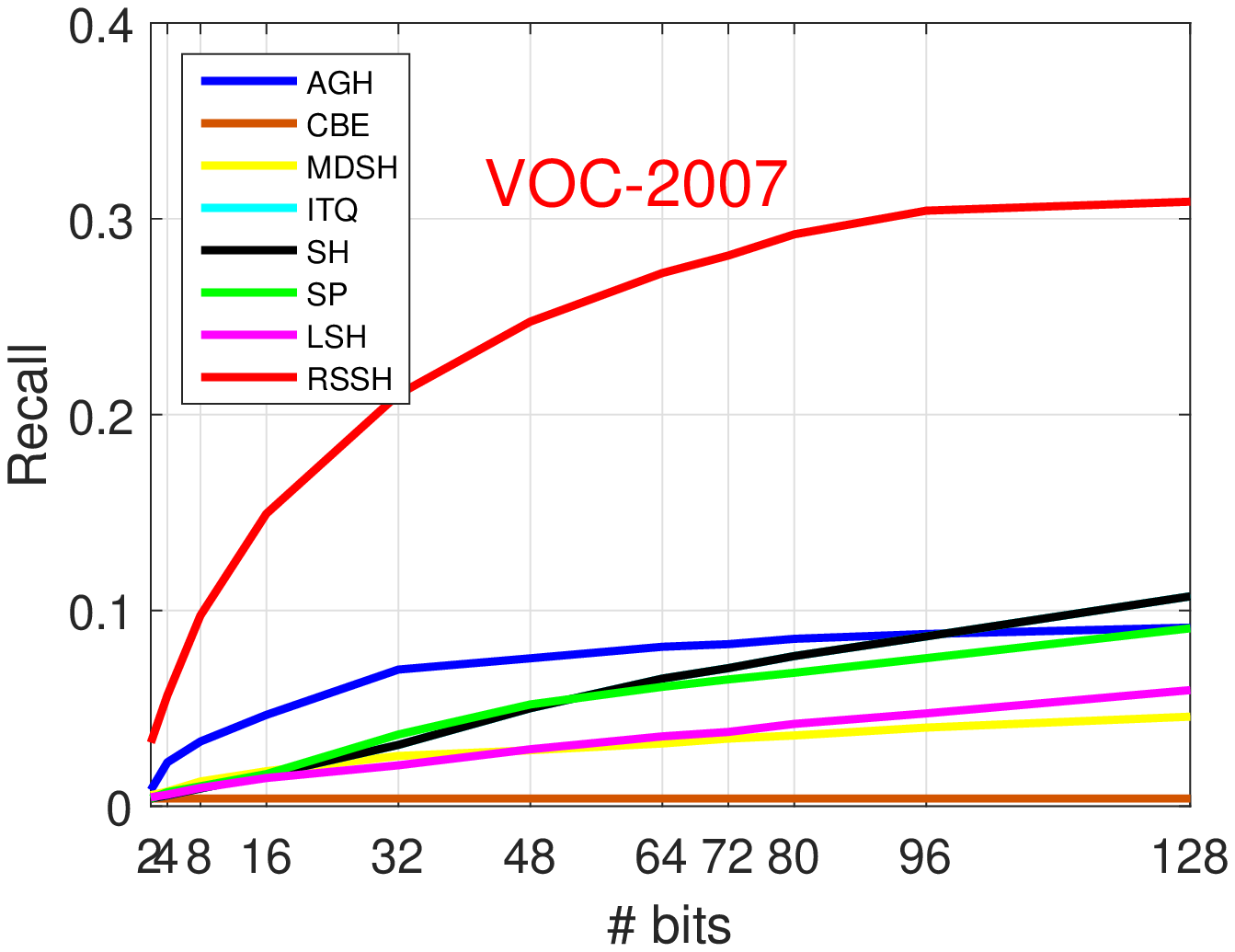}}
	\caption{The average recall in terms of the number of hash bits with different retrievals on the VOC2007.}
	\label{fig:avgrecall}
\end{figure*}
\subsection{Classification Results}
The classification accuracy on CIFAR-10, MNIST and COIL-20 are reported in Figure \ref{fig:accu}a, \ref{fig:accu}b and \ref{fig:accu}c. We can see that our algorithm achieves the best accuracy in terms of the number of hash bits in the three datasets. For example, on CIFAR-10 with $r=48$, the accuracy of our algorithm RSSN reaches 53.20\% while the comparative algorithms are all less than 40.00\%. The increase of hash bit promotes the accuracy of all algorithms, our algorithm remains the leading place. For example, on MNIST with $r=96$, ITH and SH reach the accuracy of 93.00\%, but our algorithm still enjoys 4.00\% advantage with 97.30\%.  Moreover, our algorithm obtains significant good performance even with limited information, that is, the $r$ is small. For instance, in terms of $r=8$, RSSN reaches the accuracy of 87.70\%, much better than the comparative algorithms. 

For the classification results on VOC2007, we report the accuracy on 12 of 20 classes as representation in Figure \ref{fig:precision}. We can see that it is a tough task for all the methods, but our algorithm still obtains satisfying performance. For example, on the classification task of ``horse'', our algorithm obtains around 10\% advantage over the all the comparative algorithms. The average precision on all the 20 classes are presented on Figure \ref{fig:avgprecision}. We can see that our algorithm obtains the overall best result.

\subsection{Retrieval Results}
The retrieval results on MNIST, CIFAR-10 and COIL-20 are presented in Figure \ref{fig:accu}d to Figure \ref{fig:accu}l. Our algorithm obtains the best recall with varying number of retrieved samples. For example, on CIFAR-10 with Top 10 retrieval and $r=96$, our algorithm reaches the recall over 70\%, while the others are less than 20\%. On MNIST with Top 25 retrieved samples and $r=128$, the recall of RSSN reaches 90\%, while the comparatives algorithms are around 40\%. 

\begin{figure*}
	\centering
	\subfloat[Classification accuracy on CIFAR-10]{
		\includegraphics[width=0.3\textwidth]{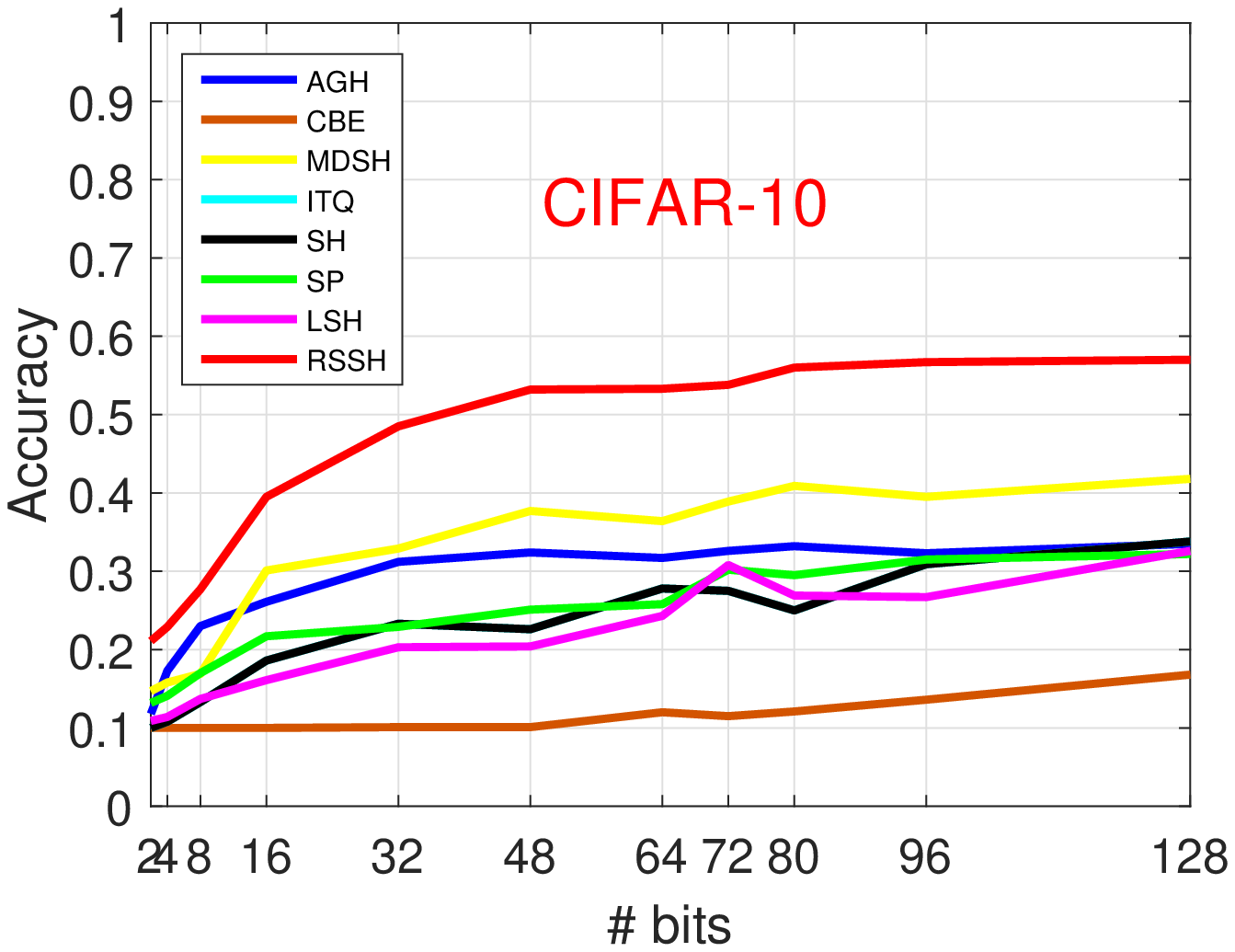}  }
	\subfloat[Classification accuracy on MNIST]{
		\includegraphics[width=0.3\textwidth]{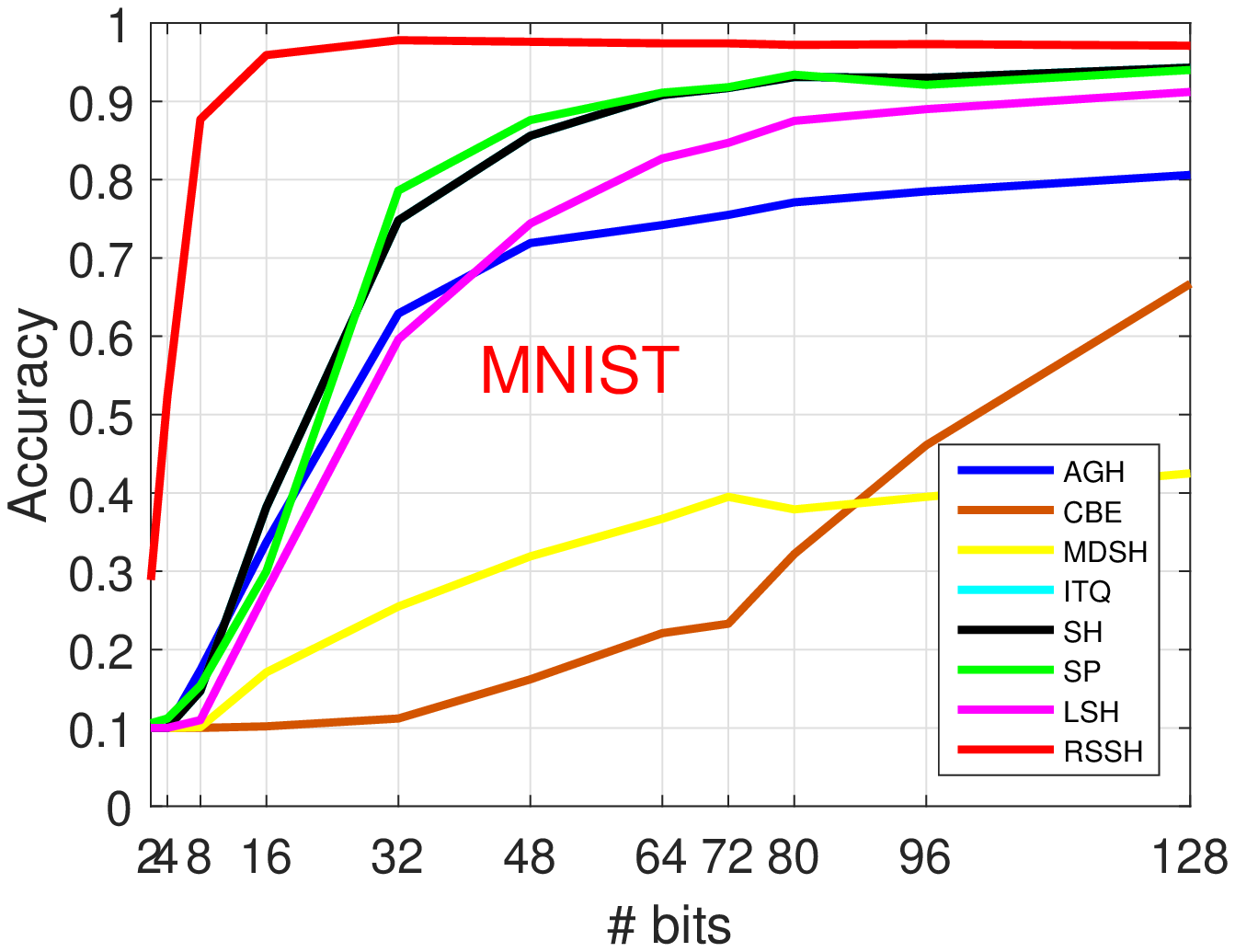}}
	\subfloat[Classification accuracy on COIL-20]{
		\includegraphics[width=0.3\textwidth]{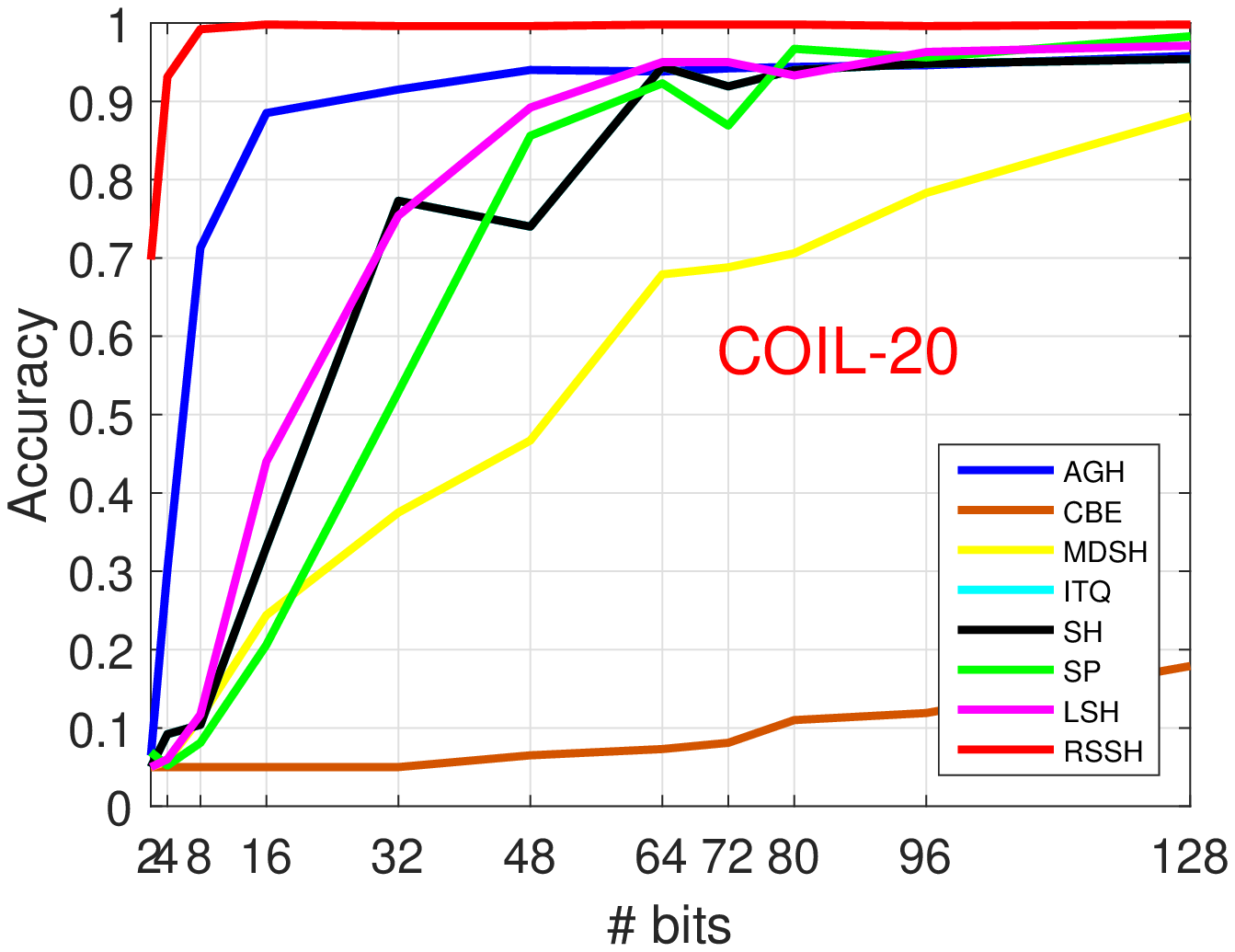}}
	
	\subfloat[Recall of Top 1 retrieval]{
		\includegraphics[width=0.3\textwidth]{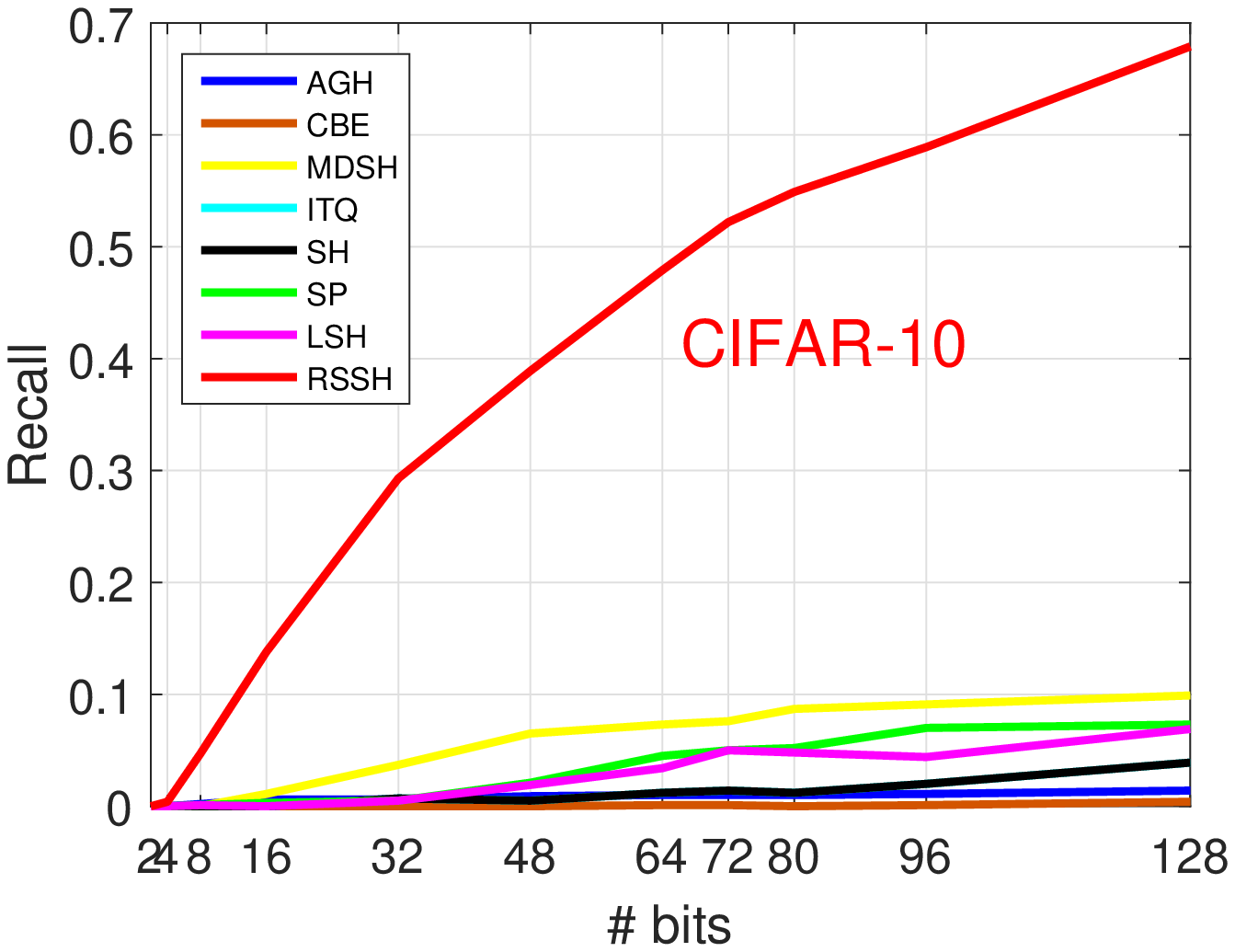}  }  
	\subfloat[Recall of Top 10 retrieval]{
		\includegraphics[width=0.3\textwidth]{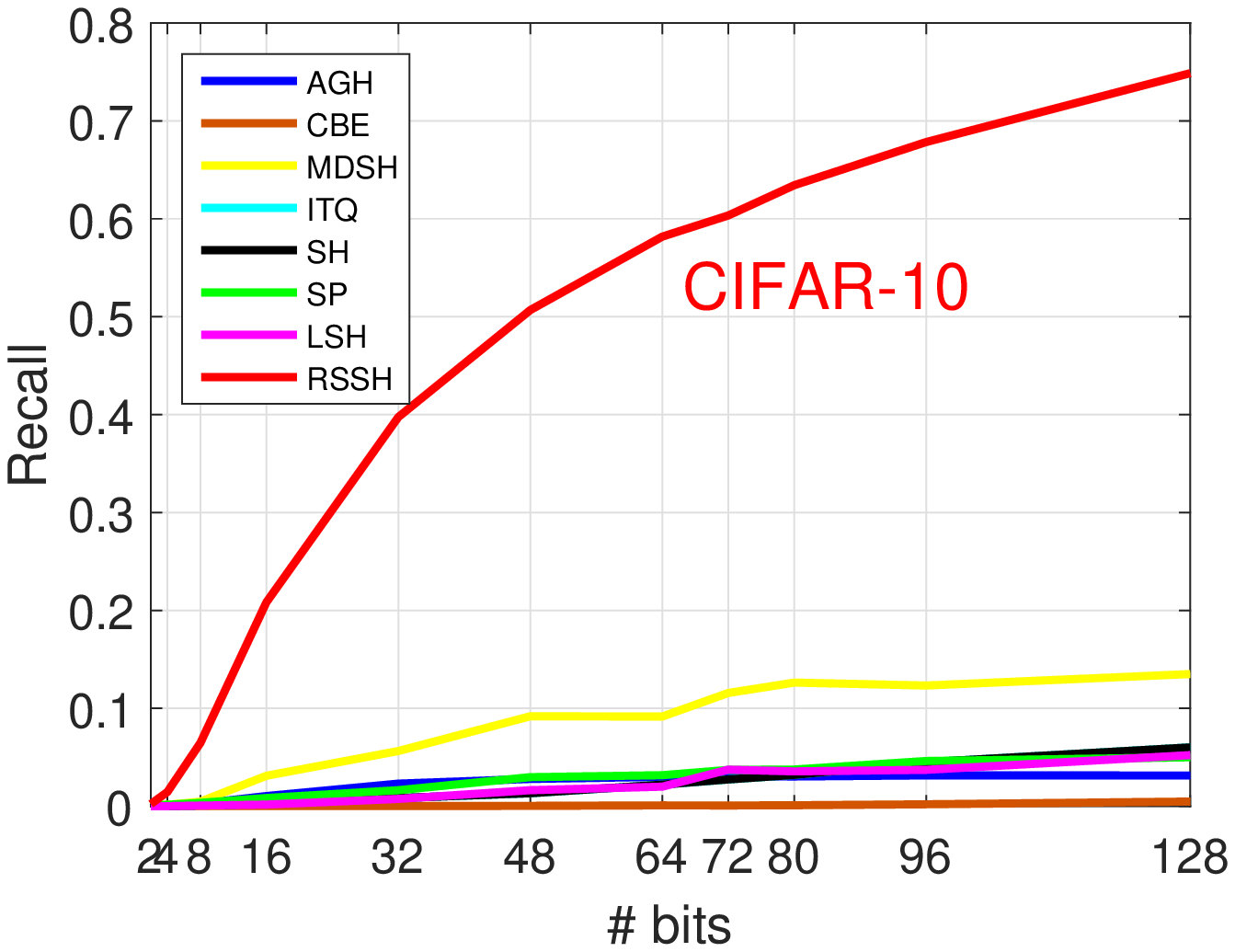}  }
	\subfloat[Recall of Top 25 retrieval]{
		\includegraphics[width=0.3\textwidth]{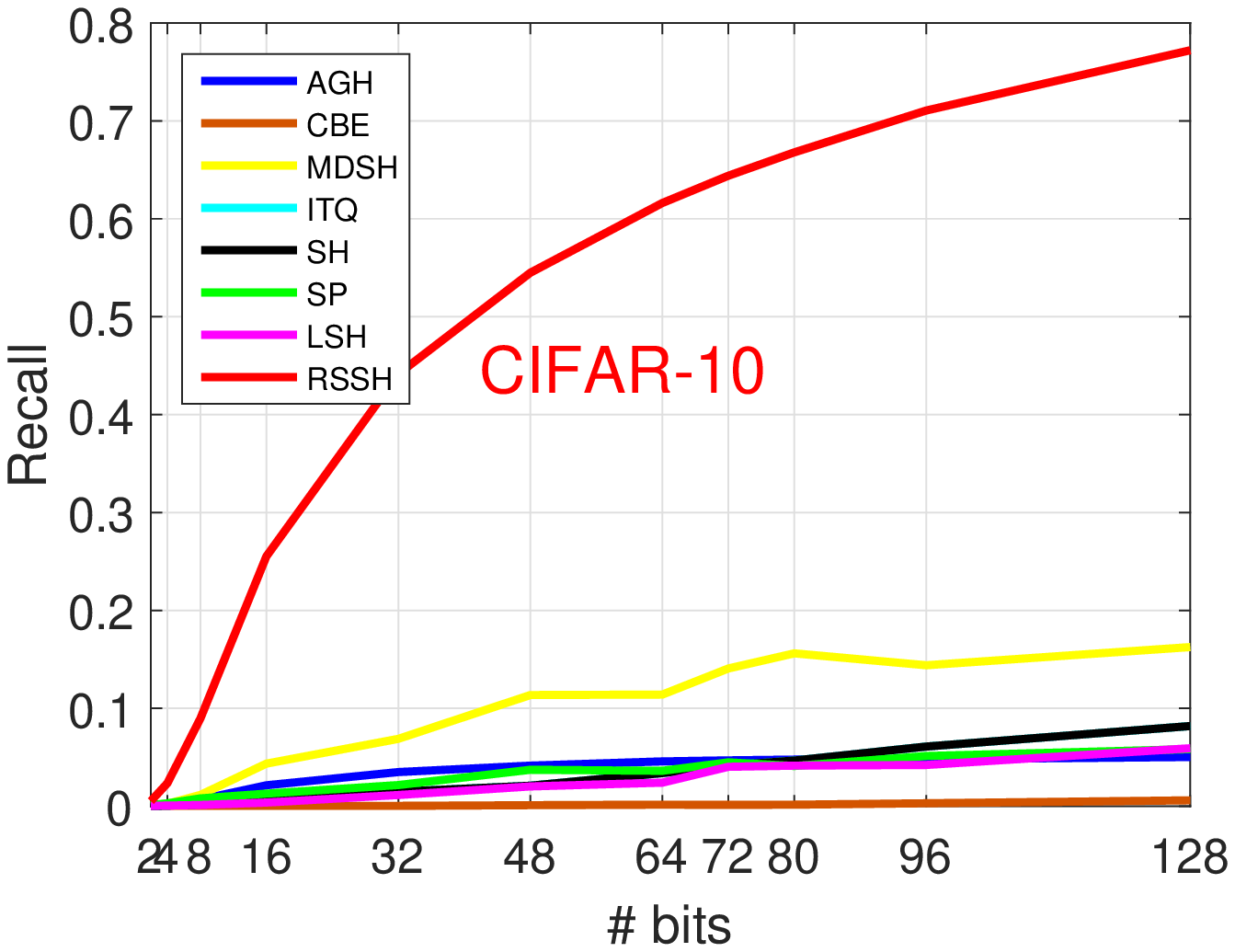}}  
	
	\subfloat[Recall of Top 1 retrieval]{           
		\includegraphics[width=0.3\textwidth]{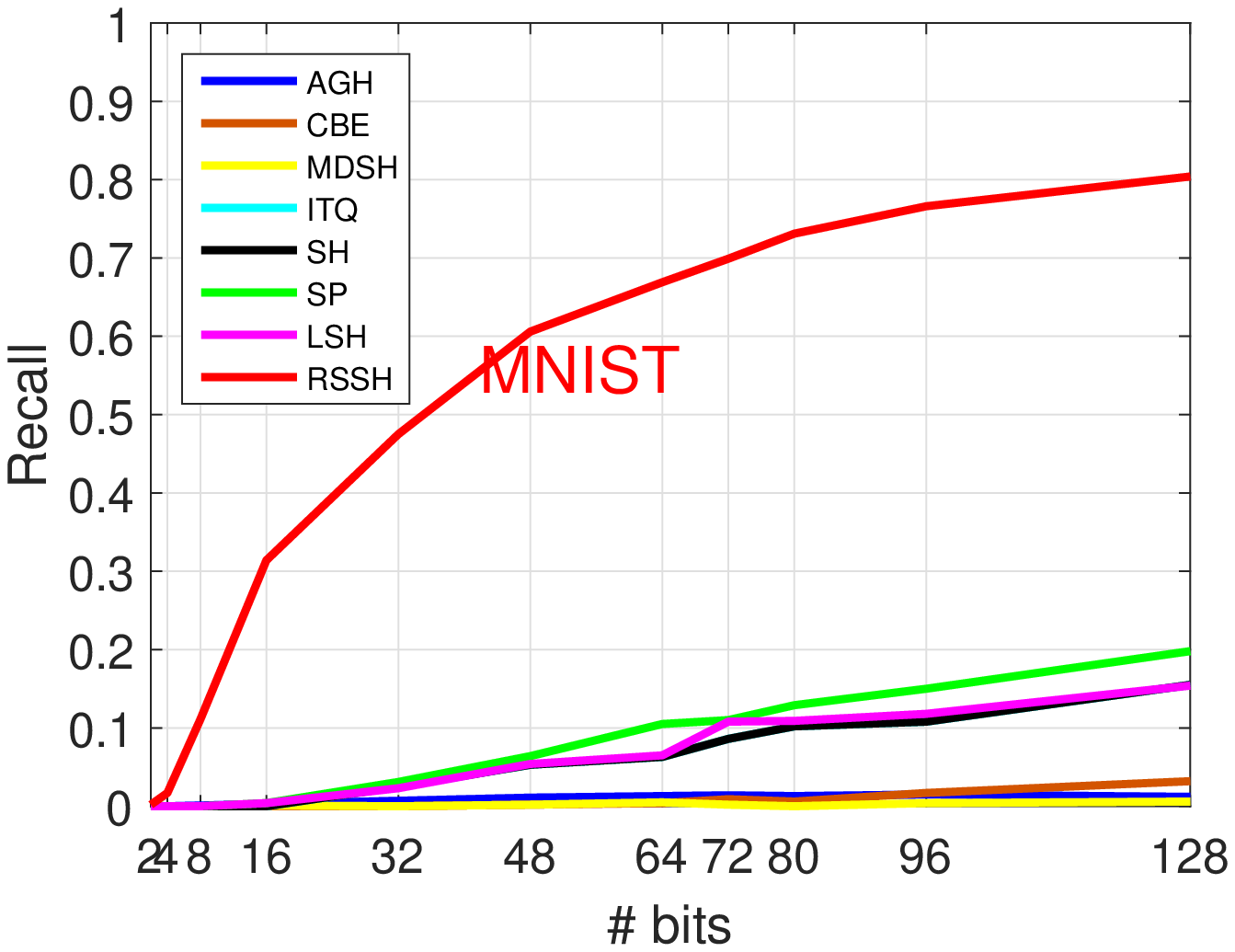}}
	\subfloat[Recall of Top 10 retrieval]{
		\includegraphics[width=0.3\textwidth]{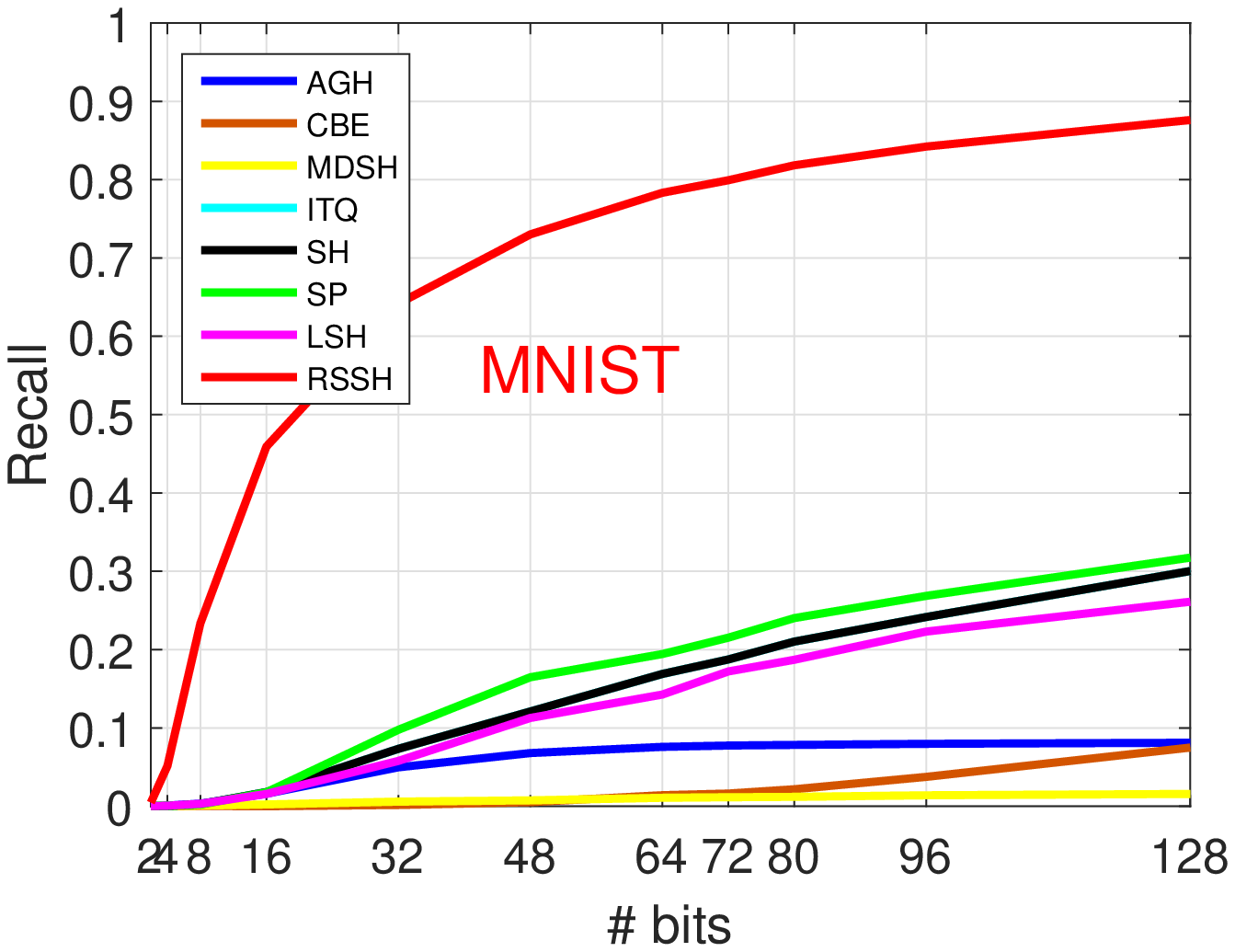}}
	\subfloat[Recall of Top 25 retrieval]{
		\includegraphics[width=0.3\textwidth]{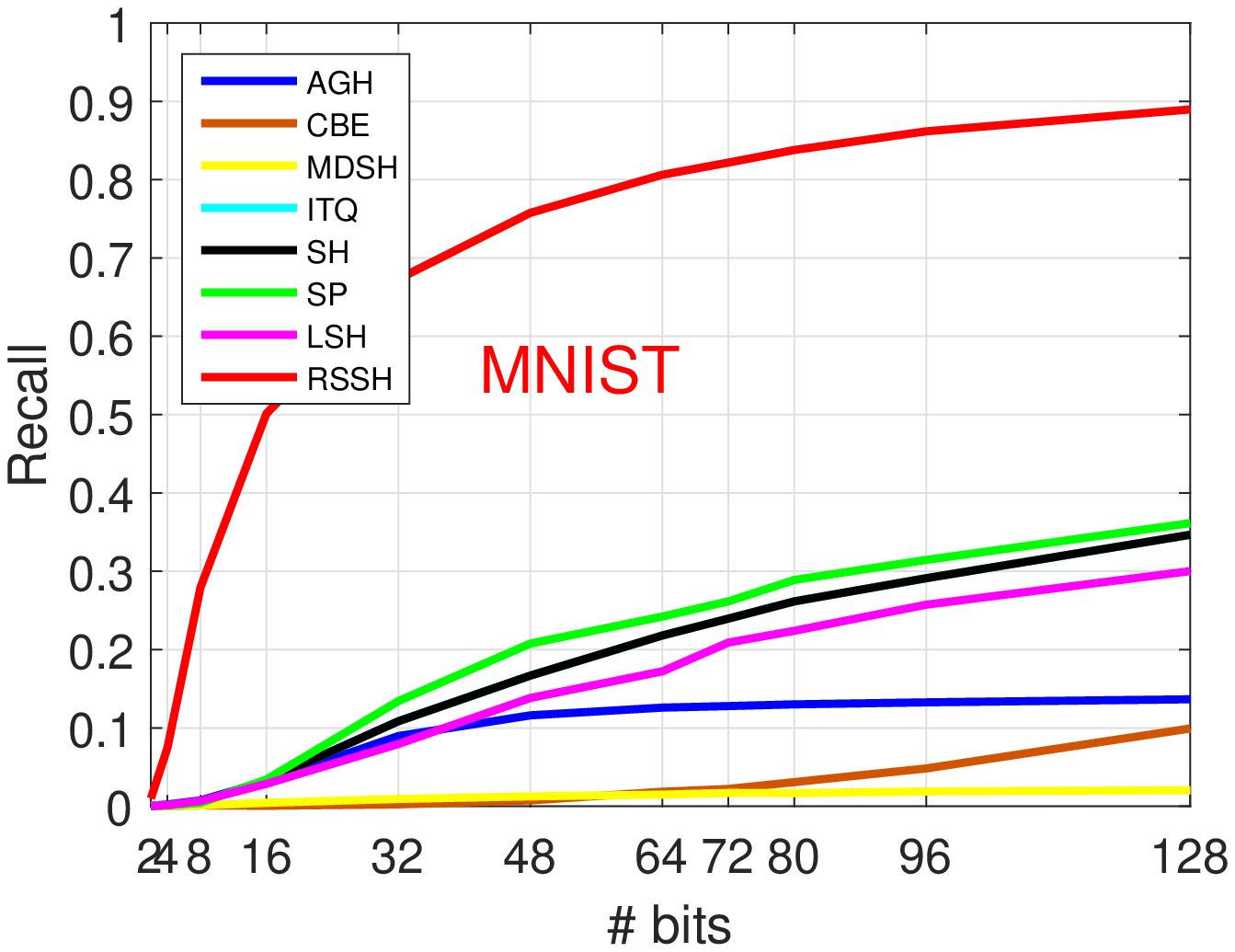}}
	
	\subfloat[Recall of Top 1 retrieval]{
		\includegraphics[width=0.3\textwidth]{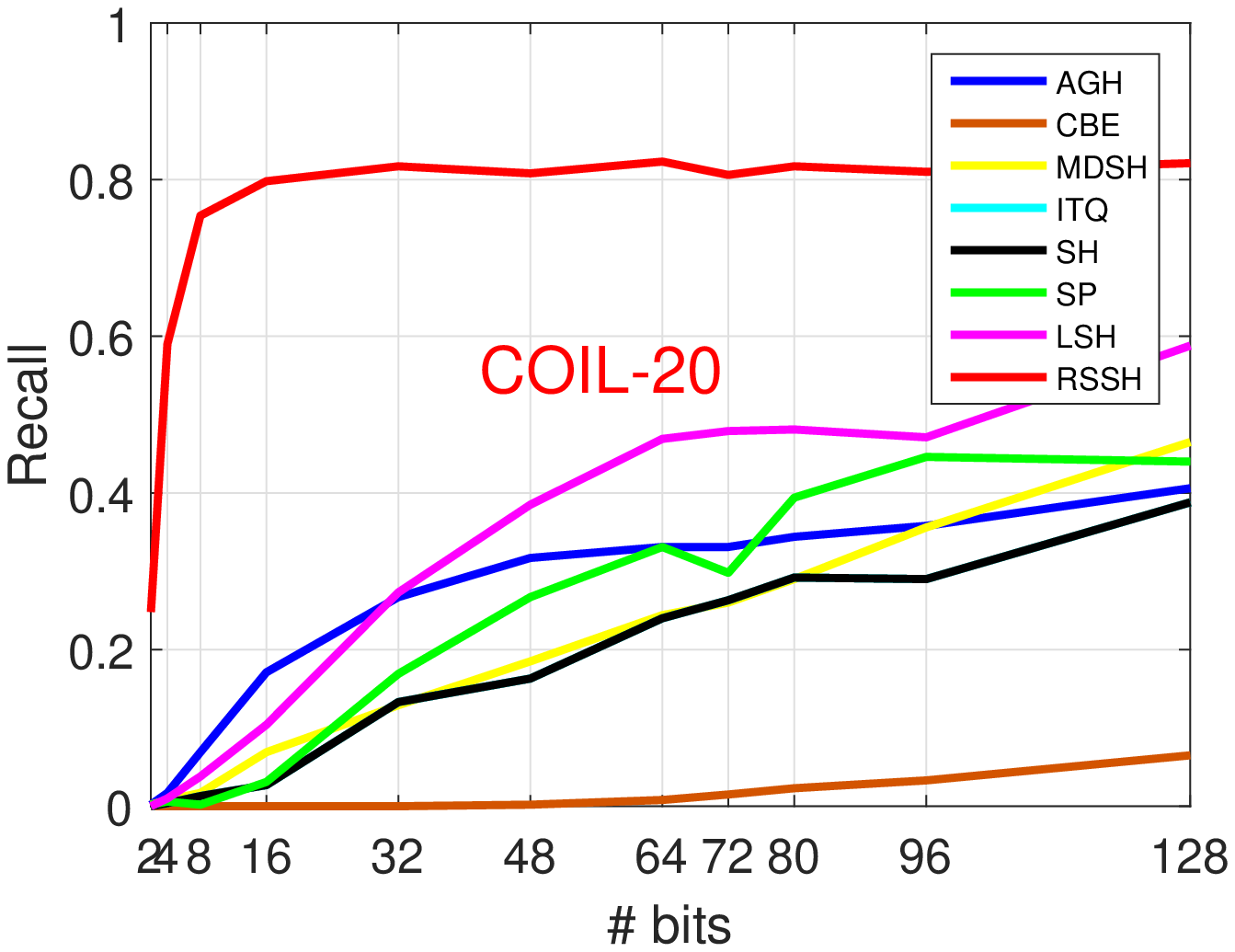}}
	\subfloat[Recall of Top 10 retrieval]{
		\includegraphics[width=0.3\textwidth]{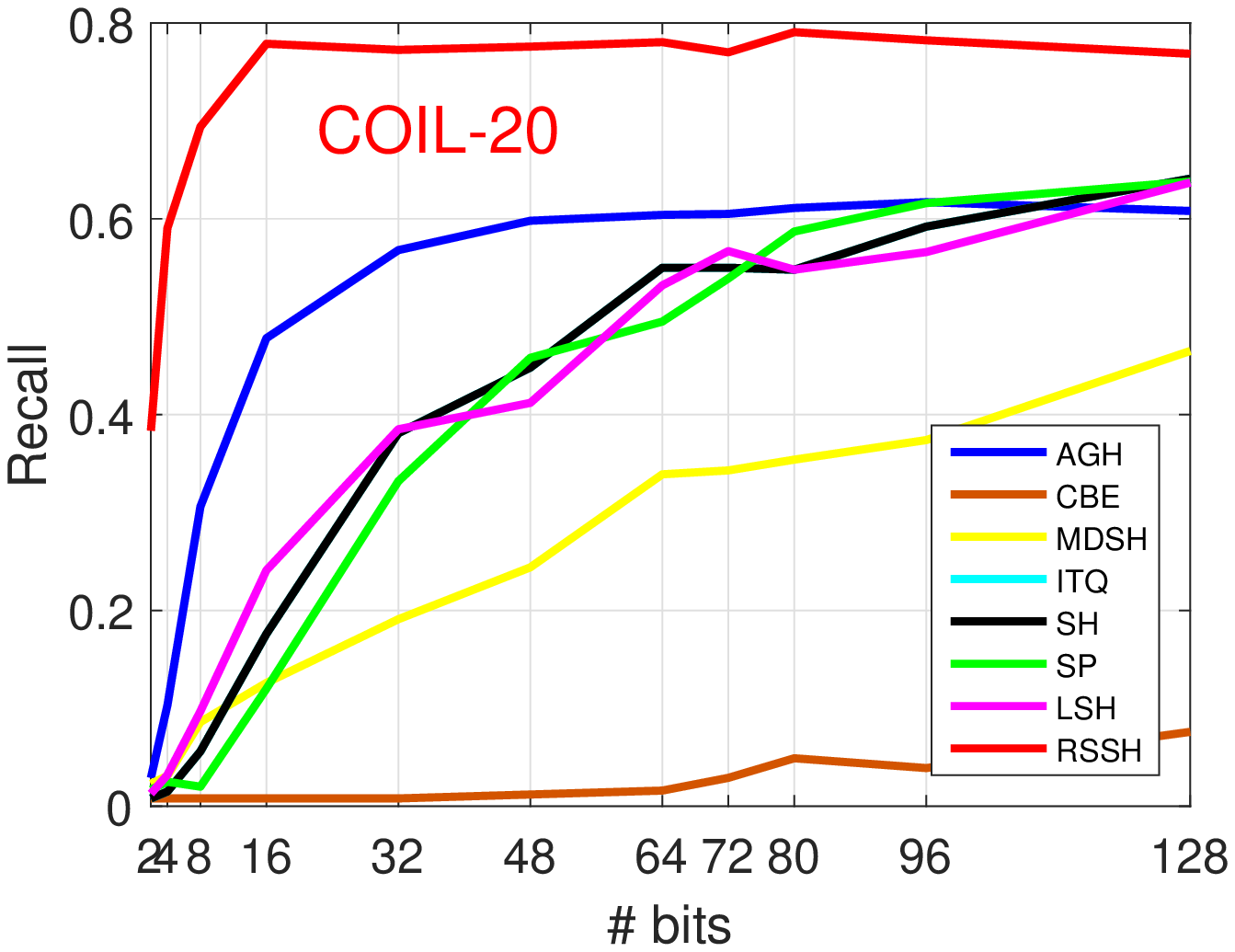}}
	\subfloat[Recall of Top 25 retrieval]{
		\includegraphics[width=0.3\textwidth]{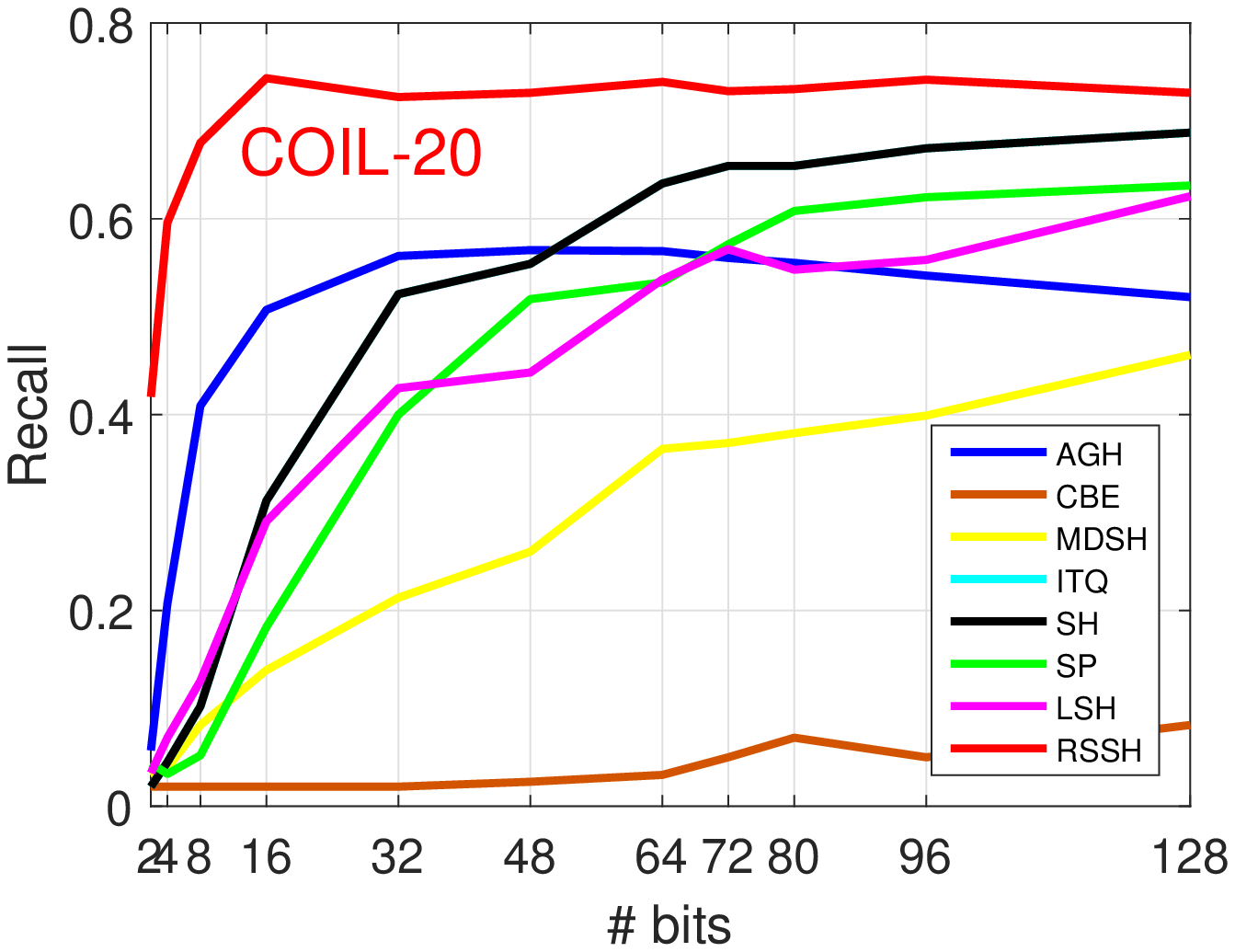}}
	\caption{Classification accuracy and recall in terms of the number of hash bits on three datasets.}
	\label{fig:accu}
\end{figure*}

\begin{figure*}
	\centering
	\subfloat[Precision on the class: Aeroplane]{
		\includegraphics[width=0.3\textwidth]{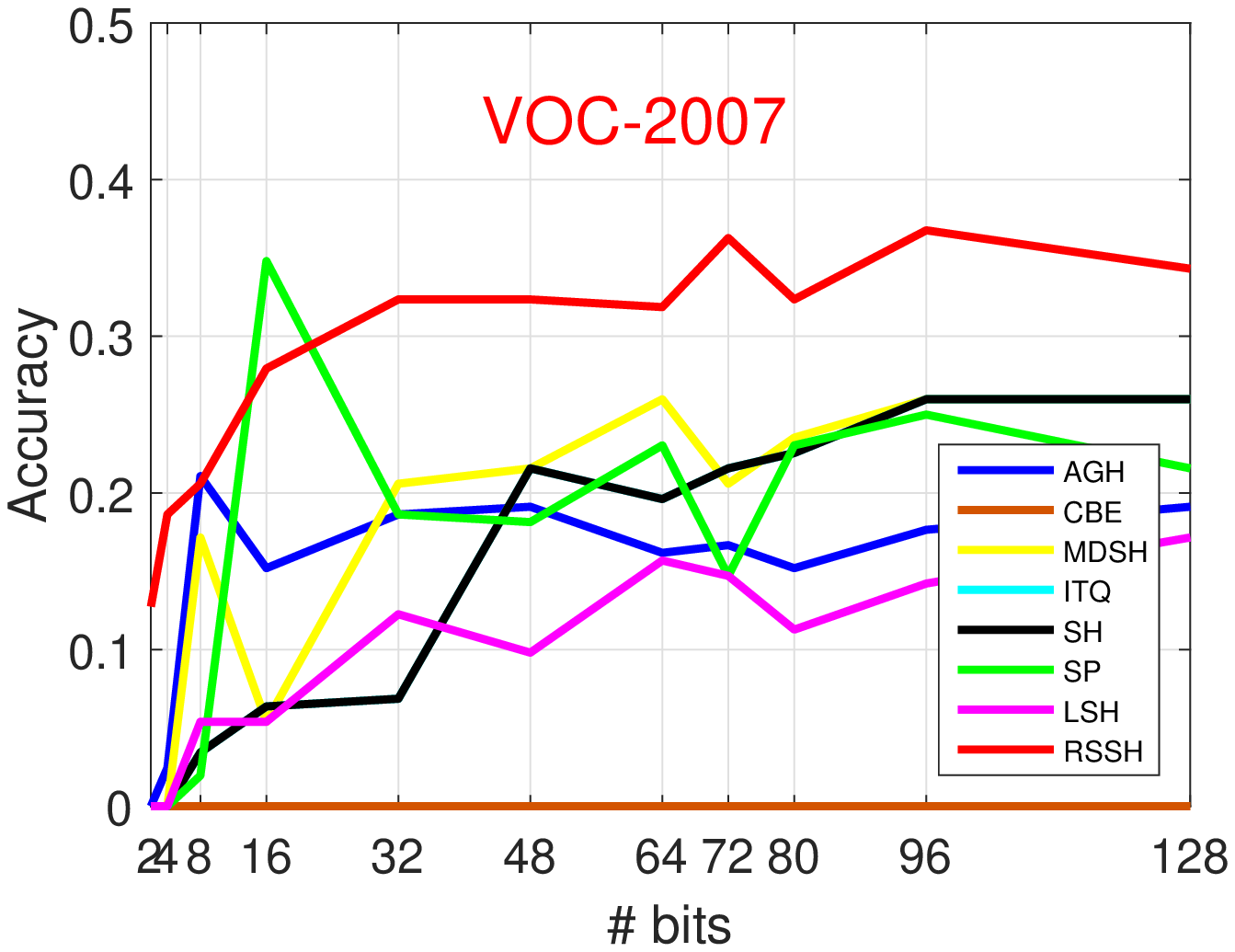}  }
	\subfloat[Precision on the class: Bicycle]{
		\includegraphics[width=0.3\textwidth]{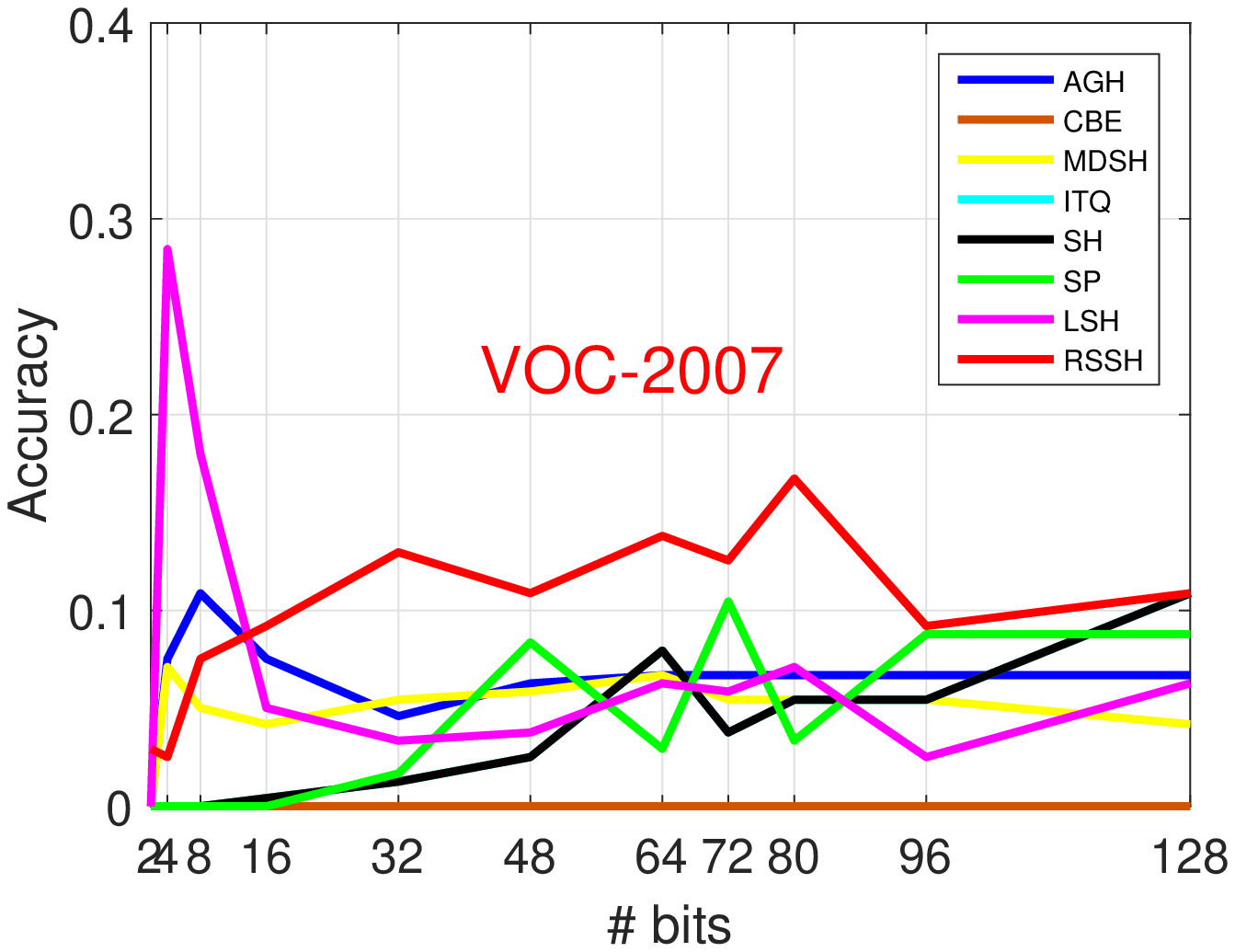}}
	\subfloat[Precision on the class: Bird]{
		\includegraphics[width=0.3\textwidth]{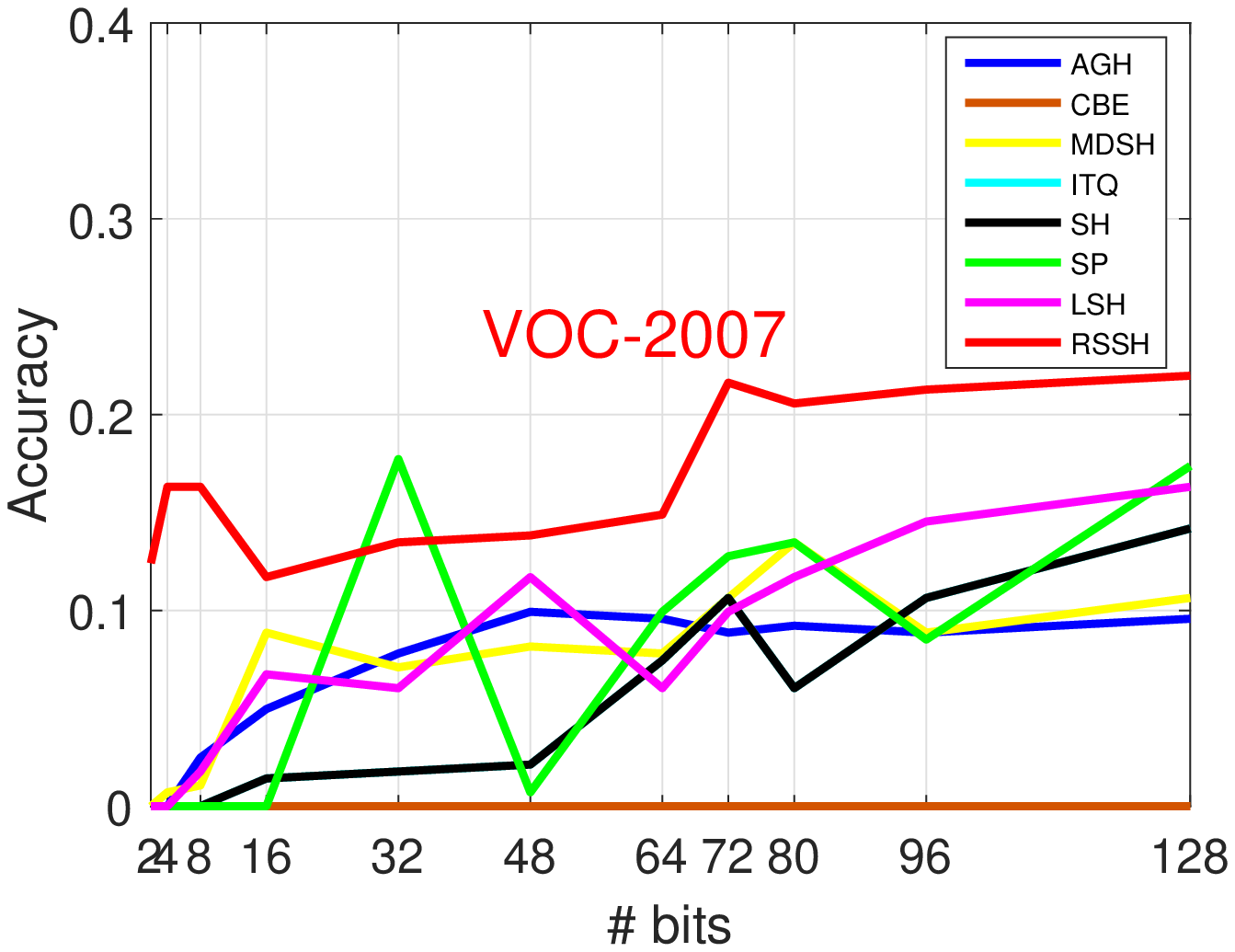}}
	
	\subfloat[Precision on the class: Bus]{
		\includegraphics[width=0.3\textwidth]{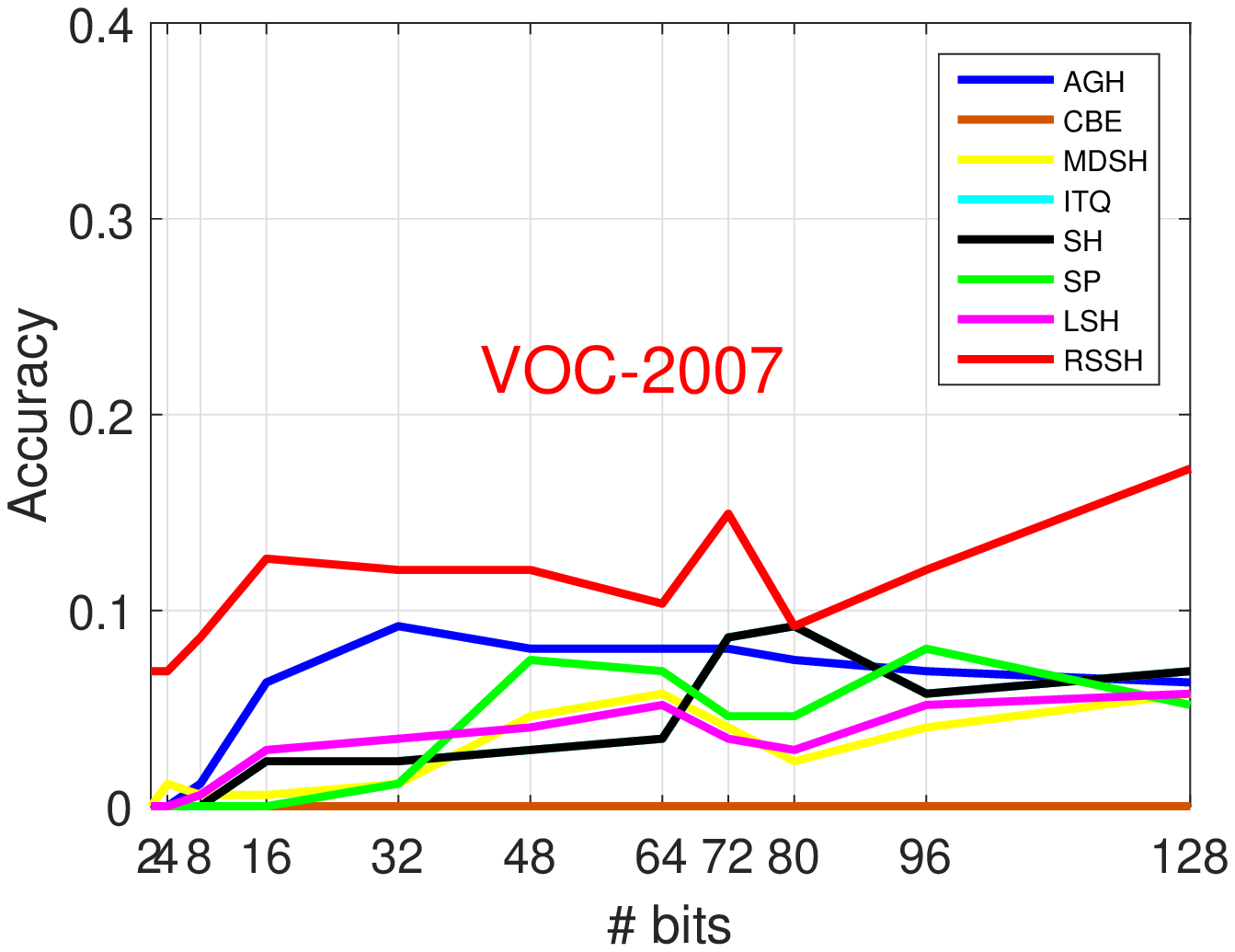}  }\subfloat[Precision on the class: ]{
		\includegraphics[width=0.3\textwidth]{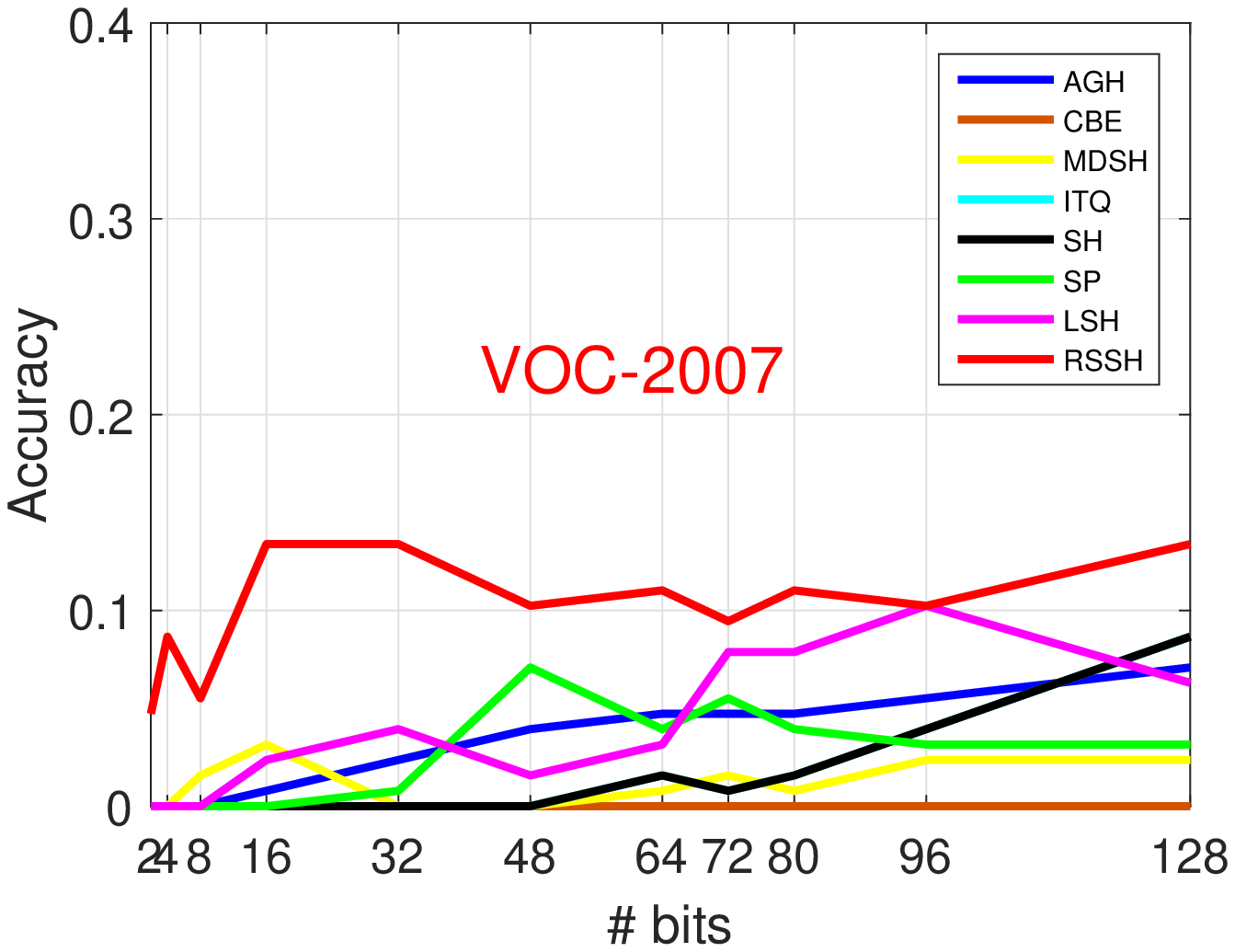}  }
	\subfloat[Precision on the class: Cow]{
		\includegraphics[width=0.3\textwidth]{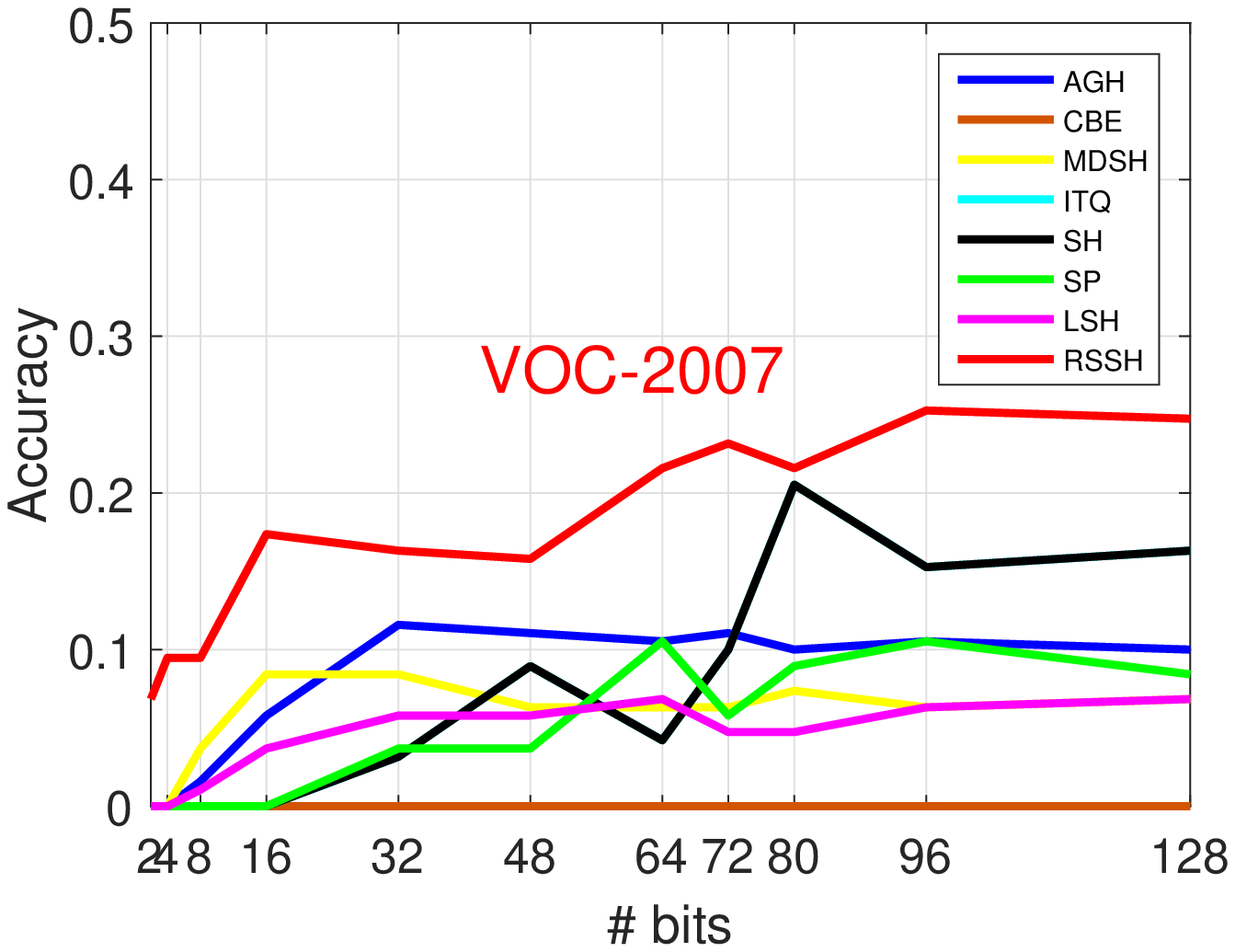}}  
	
	\subfloat[Precision on the class: Dining table]{
		\includegraphics[width=0.3\textwidth]{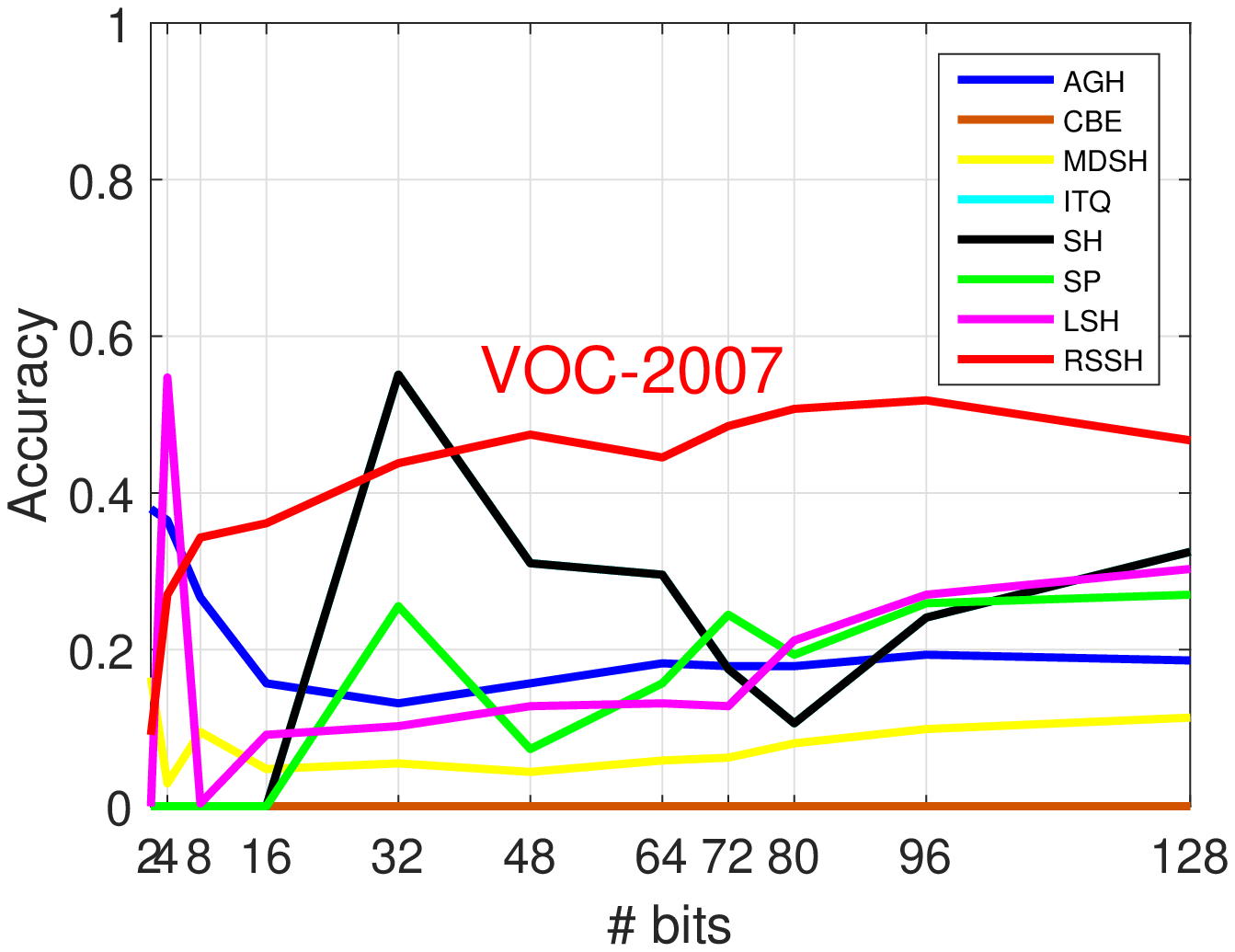}}
	\subfloat[Precision on the class: Horse]{
		\includegraphics[width=0.3\textwidth]{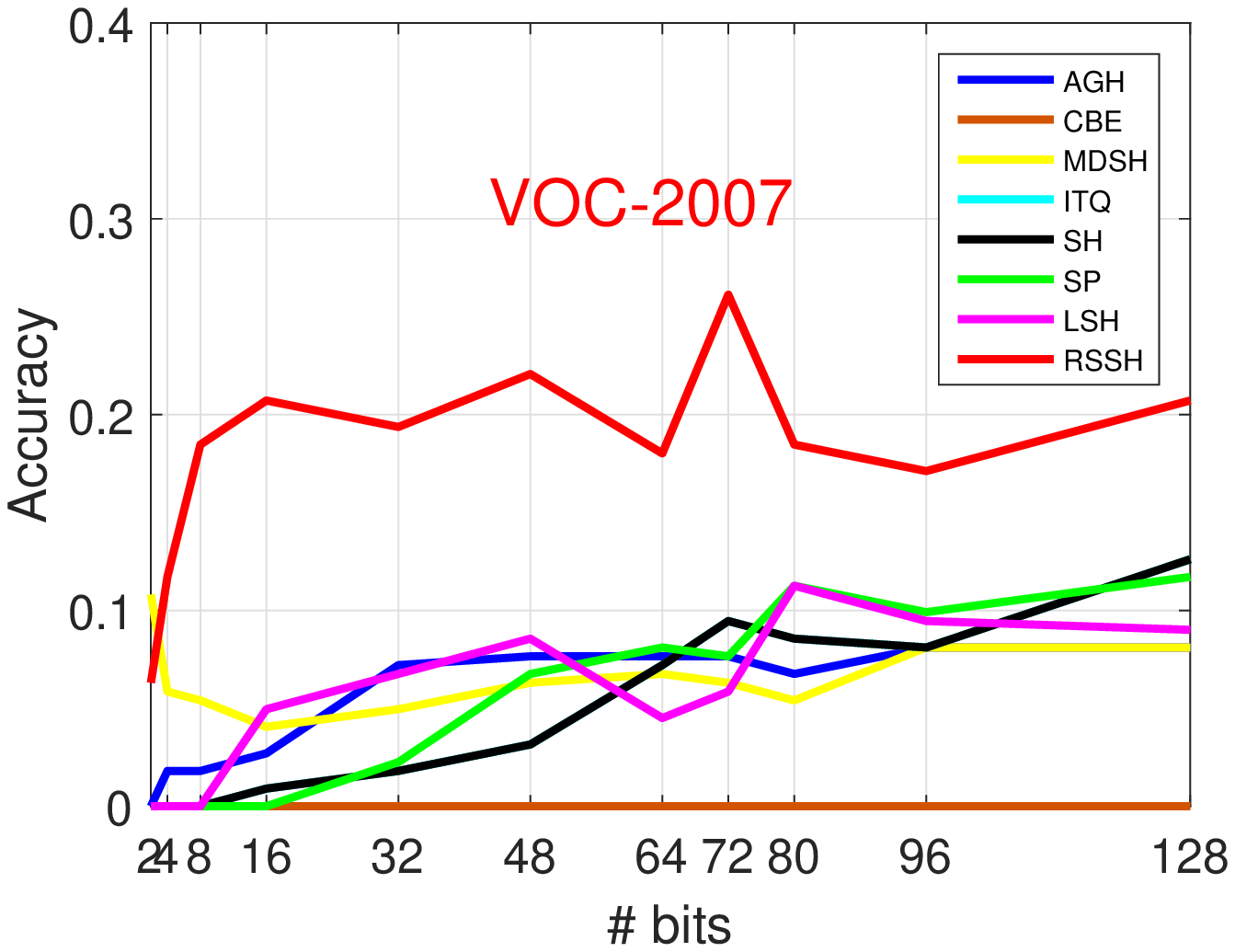}}
	\subfloat[Precision on the class: Motorbike]{
		\includegraphics[width=0.3\textwidth]{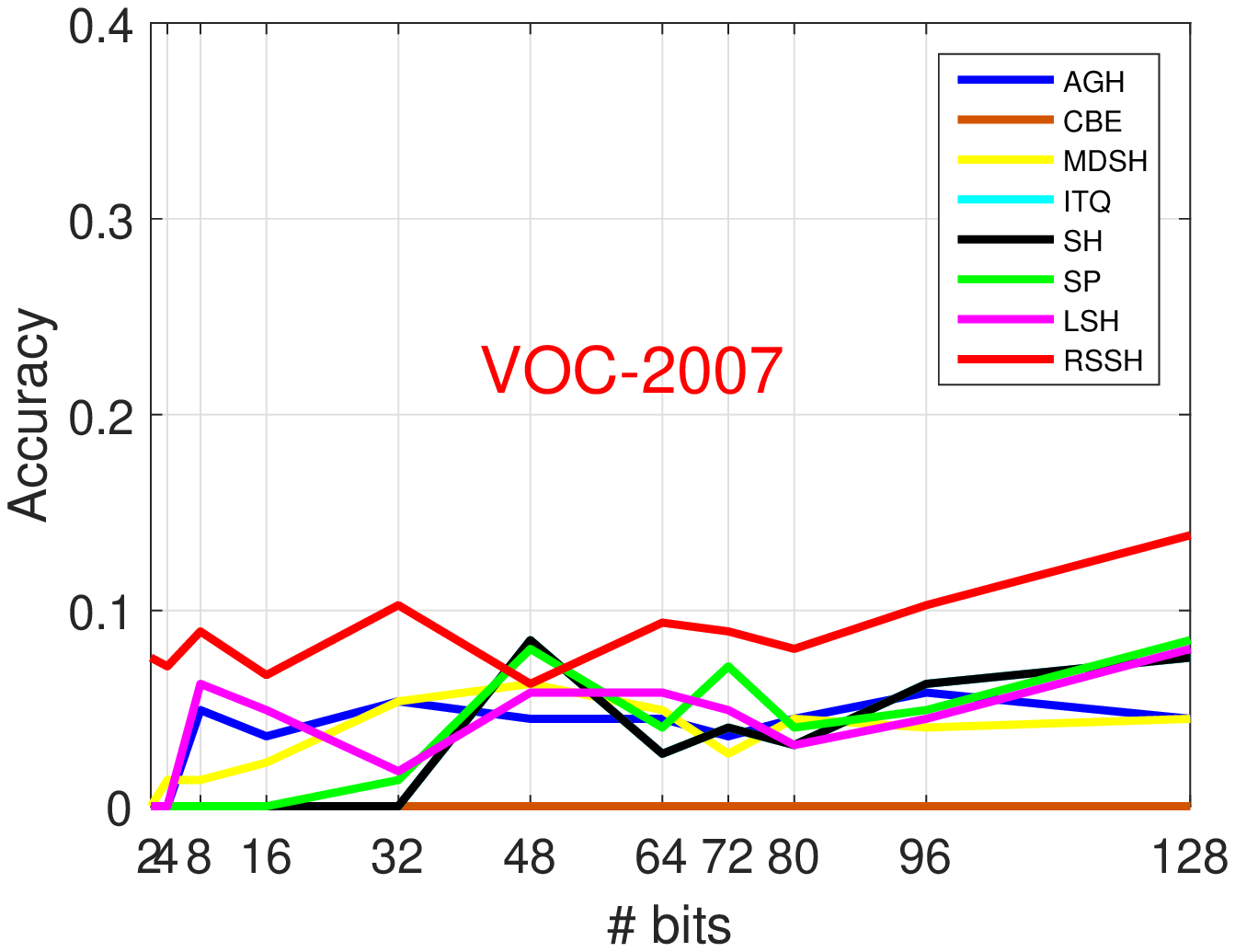}}
	
	\subfloat[Precision on the class: Potted plant]{
		\includegraphics[width=0.3\textwidth]{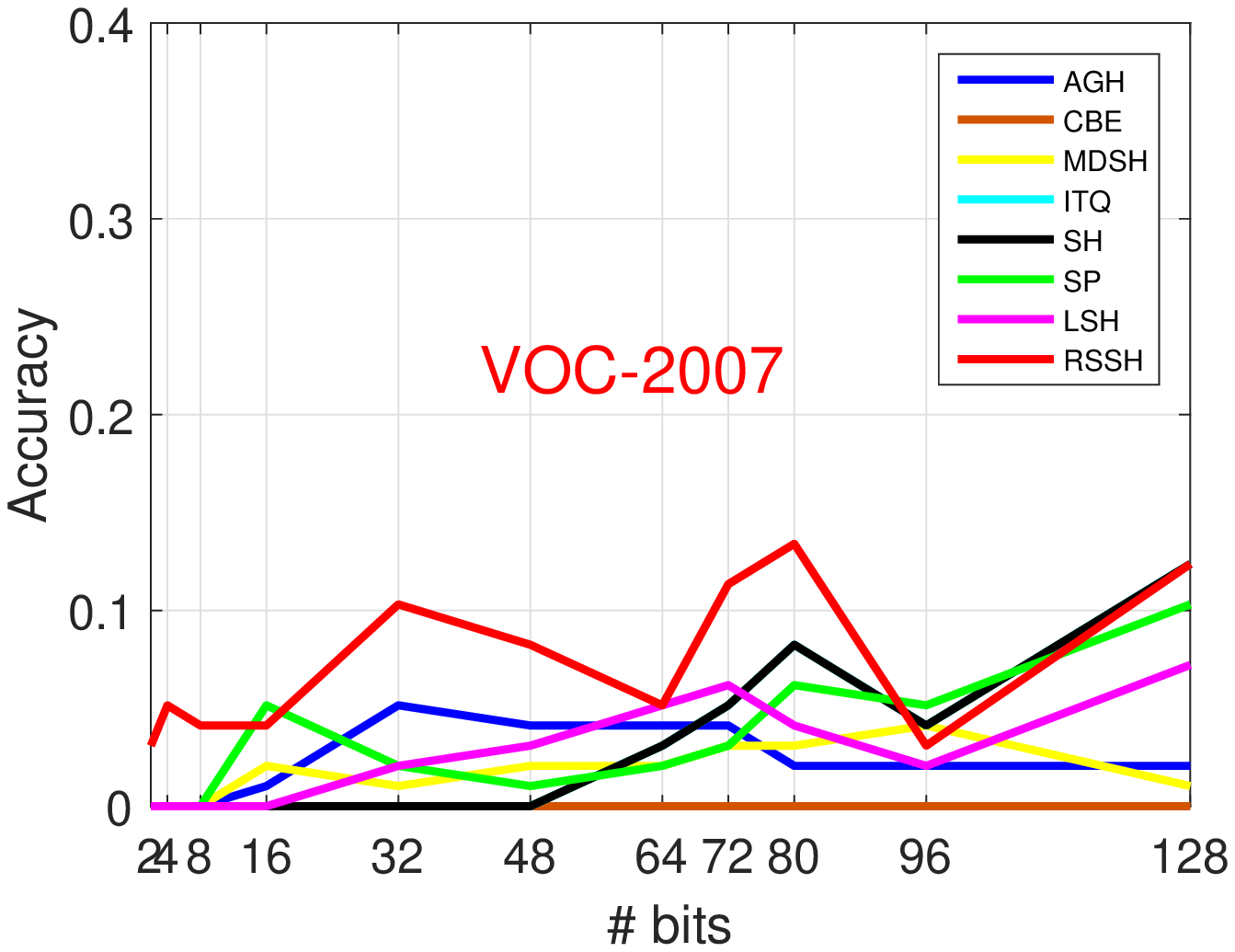}}
	\subfloat[Precision on the class: Train]{
		\includegraphics[width=0.3\textwidth]{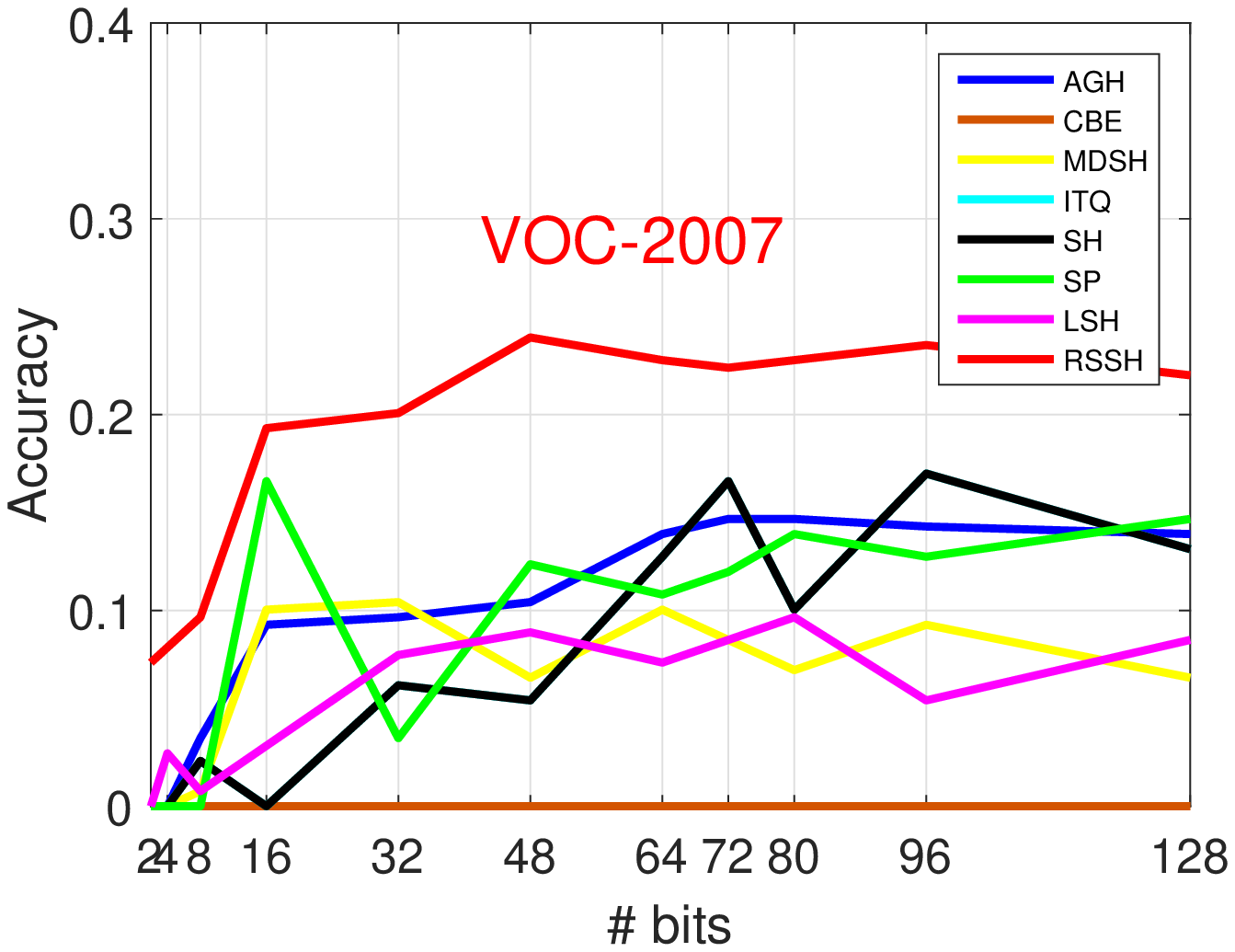}}
	\subfloat[Precision on the class: Tv monitor]{
		\includegraphics[width=0.3\textwidth]{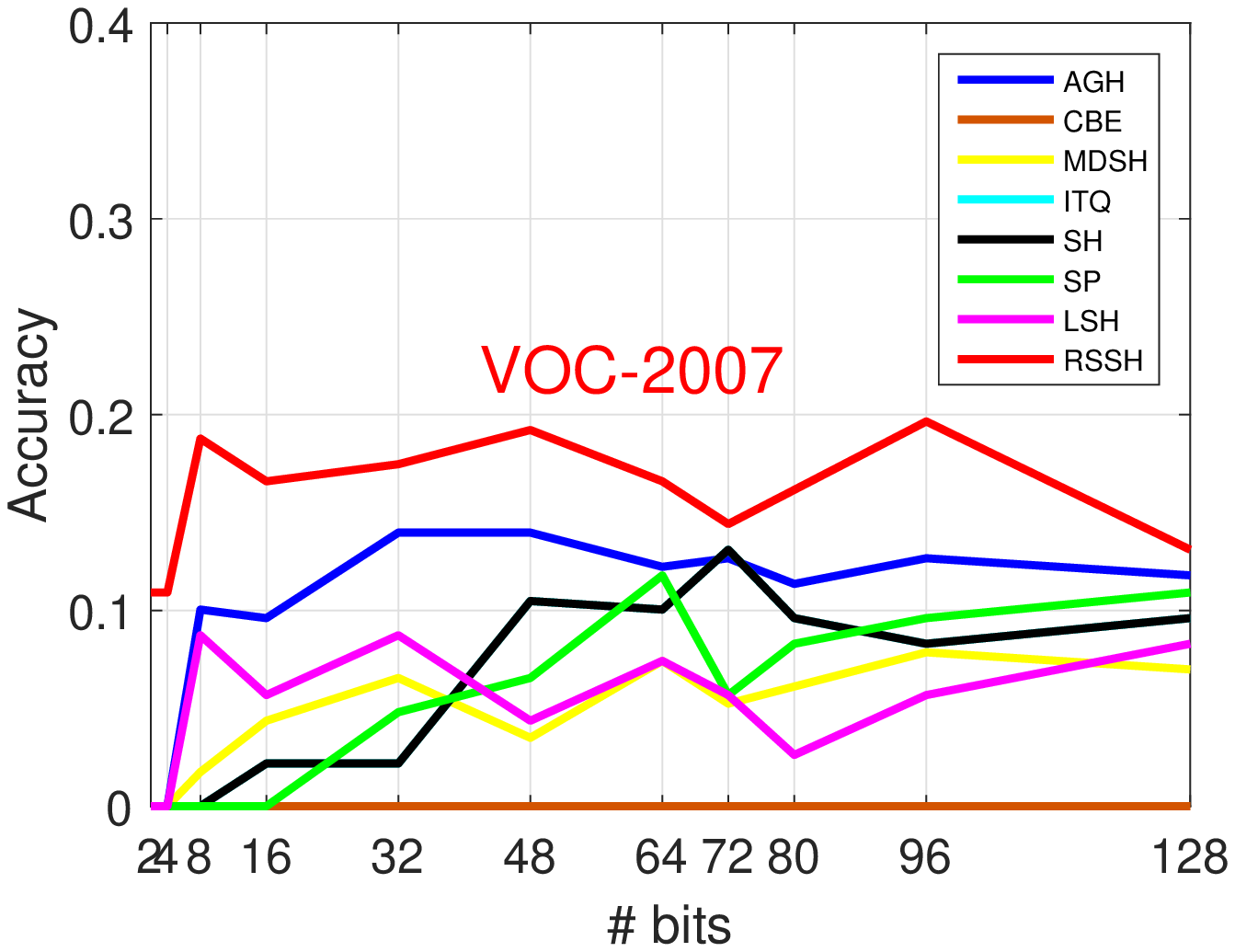}}
	\caption{Precision and recall in terms of the number of hash bits on various classes of VOC2007.}
	\label{fig:precision}
\end{figure*}

\clearpage

\bibliographystyle{apalike}
\bibliography{nnsnlp} 

\begin{thebibliography}{}

\bibitem[Abdullah et~al., 2014]{abdullah2014spectral}
Abdullah, A., Andoni, A., Kannan, R., and Krauthgamer, R. (2014).
\newblock Spectral approaches to nearest neighbor search.
\newblock In {\em IEEE 55th Annual Symposium on Foundations of Computer
  Science}, pages 581--590. IEEE.

\bibitem[Andreas et~al., 2016]{andreas2016learning}
Andreas, J., Rohrbach, M., Darrell, T., and Klein, D. (2016).
\newblock Learning to compose neural networks for question answering.
\newblock {\em arXiv preprint arXiv:1601.01705}.

\bibitem[Gong and Lazebnik, 2011]{GongL11}
Gong, Y. and Lazebnik, S. (2011).
\newblock Iterative quantization: {A} procrustean approach to learning binary
  codes.
\newblock In {\em Proceedings of the 24th IEEE Conference on Computer Vision
  and Pattern Recognition}, pages 817--824.

\bibitem[Kumar et~al., 2016]{kumar2016ask}
Kumar, A., Irsoy, O., Ondruska, P., Iyyer, M., Bradbury, J., Gulrajani, I.,
  Zhong, V., Paulus, R., and Socher, R. (2016).
\newblock Ask me anything: Dynamic memory networks for natural language
  processing.
\newblock In {\em International Conference on Machine Learning}, pages
  1378--1387.

\bibitem[Liu et~al., 2011a]{liu2011hashing}
Liu, W., Wang, J., Kumar, S., and Chang, S.-F. (2011a).
\newblock Hashing with graphs.
\newblock In {\em Proceedings of the 28th International Conference on Machine
  Learning}, pages 1--8.

\bibitem[Liu et~al., 2011b]{liu2011recognizing}
Liu, X., Zhang, S., Wei, F., and Zhou, M. (2011b).
\newblock Recognizing named entities in tweets.
\newblock In {\em Proceedings of the 49th Annual Meeting of the Association for
  Computational Linguistics}, pages 359--367. Association for Computational
  Linguistics.

\bibitem[Mikolov et~al., 2013]{mikolov2013linguistic}
Mikolov, T., Yih, W.-t., and Zweig, G. (2013).
\newblock Linguistic regularities in continuous space word representations.
\newblock In {\em hlt-Naacl}, volume~13, pages 746--751.

\bibitem[Musco and Musco, 2015]{musco2015randomized}
Musco, C. and Musco, C. (2015).
\newblock Randomized block krylov methods for stronger and faster approximate
  singular value decomposition.
\newblock In {\em Advances in Neural Information Processing Systems}, pages
  1396--1404.

\bibitem[Passos et~al., 2014]{passos2014lexicon}
Passos, A., Kumar, V., and McCallum, A. (2014).
\newblock Lexicon infused phrase embeddings for named entity resolution.
\newblock In {\em Proceedings of the Eighteenth Conference on Computational
  Language Learning}, pages 78--86.

\bibitem[Ramanathan et~al., 2014]{ramanathan2014linking}
Ramanathan, V., Joulin, A., Liang, P., and Fei-Fei, L. (2014).
\newblock Linking people in videos with “their” names using coreference
  resolution.
\newblock In {\em European Conference on Computer Vision}, pages 95--110.

\bibitem[Socher et~al., 2013]{socher2013recursive}
Socher, R., Perelygin, A., Wu, J.~Y., Chuang, J., Manning, C.~D., Ng, A.~Y.,
  Potts, C., et~al. (2013).
\newblock Recursive deep models for semantic compositionality over a sentiment
  treebank.
\newblock In {\em Proceedings of the Conference on Empirical Methods in Natural
  Language Processing}, volume 1631, page 1642.

\bibitem[Weiss et~al., 2012]{WeissFT12}
Weiss, Y., Fergus, R., and Torralba, A. (2012).
\newblock Multidimensional spectral hashing.
\newblock In {\em Proceedings of the 12th European Conference on Computer
  Vision}, pages 340--353.

\bibitem[Weiss et~al., 2009]{weiss2009spectral}
Weiss, Y., Torralba, A., and Fergus, R. (2009).
\newblock Spectral hashing.
\newblock In {\em Proceedings of the 23rd Annual Conference on Neural
  Information Processing Systems}, pages 1753--1760.

\bibitem[Wu et~al., 2016]{wu2016google}
Wu, Y., Schuster, M., Chen, Z., Le, Q.~V., Norouzi, M., Macherey, W., Krikun,
  M., Cao, Y., Gao, Q., Macherey, K., et~al. (2016).
\newblock Google's neural machine translation system: Bridging the gap between
  human and machine translation.
\newblock {\em arXiv preprint arXiv:1609.08144}.

\bibitem[Xia et~al., 2015]{XiaHK015}
Xia, Y., He, K., Kohli, P., and Sun, J. (2015).
\newblock Sparse projections for high-dimensional binary codes.
\newblock In {\em Proceedings of the 28th IEEE Conference on Computer Vision
  and Pattern Recognition}, pages 3332--3339.

\bibitem[Yates et~al., 2007]{yates2007textrunner}
Yates, A., Cafarella, M., Banko, M., Etzioni, O., Broadhead, M., and Soderland,
  S. (2007).
\newblock Textrunner: open information extraction on the web.
\newblock In {\em Proceedings of Human Language Technologies: The Annual
  Conference of the North American Chapter of the Association for Computational
  Linguistics: Demonstrations}, pages 25--26.

\bibitem[Yu et~al., 2014]{YuKGC14}
Yu, F.~X., Kumar, S., Gong, Y., and Chang, S. (2014).
\newblock Circulant binary embedding.
\newblock In {\em Proceedings of the 31th International Conference on Machine
  Learning}, pages 946--954.

\end{thebibliography}

\end{document}